\documentclass{article}
\usepackage{silence}
\usepackage[subtle]{savetrees}
 \usepackage[preprint]{neurips_2026}

\usepackage[utf8]{inputenc}
\usepackage[T1]{fontenc}
\usepackage{url}
\usepackage{microtype}
\usepackage[dvipsnames]{xcolor}
\usepackage{comment}
\usepackage[title]{appendix}

\usepackage{graphicx}
\usepackage{wrapfig}
\usepackage{subcaption}
\usepackage{float}
\usepackage[font=small,labelfont=bf]{caption}
\usepackage[normalem]{ulem}
\usepackage{placeins}
\usepackage{booktabs}
\usepackage{multicol}
\usepackage{multirow}
\usepackage{tablefootnote}
\usepackage{enumitem}
\usepackage{chngcntr}   % appendix figure numbering

\usepackage{amsmath}
\usepackage{amsfonts}
\usepackage{amssymb}
\usepackage{mathtools}
\usepackage{mathrsfs}
\usepackage{nicefrac}
\usepackage{amsthm}
\usepackage{thmtools}
\usepackage{thm-restate}
\usepackage{algpseudocode}
\usepackage{array, makecell}
\usepackage{listings}
\usepackage{amsmath,amsfonts,bm}

\def\eqref#1{equation~\ref{#1}}
\def\Eqref#1{Equation~\ref{#1}}
\def\1{\bm{1}}

\DeclareMathAlphabet{\mathsfit}{\encodingdefault}{\sfdefault}{m}{sl}
\SetMathAlphabet{\mathsfit}{bold}{\encodingdefault}{\sfdefault}{bx}{n}

\newcommand{\KL}{D_{\mathrm{KL}}}

\declaretheorem{theorem}

\declaretheorem[numberwithin=theorem]{lemma}
\declaretheorem[numberwithin=theorem]{corollary}

\declaretheorem[style=definition]{definition}

\usepackage{hyperref}
\usepackage[nameinlink,capitalize,noabbrev]{cleveref}
\definecolor{mydarkblue}{rgb}{0,0.08,0.45}
\hypersetup{colorlinks=true,
    linkcolor=mydarkblue,
    citecolor=mydarkblue,
    filecolor=mydarkblue,
    urlcolor=mydarkblue
}

\usepackage[most]{tcolorbox}
\newtcolorbox[auto counter]{mybox}[2][]
{
  center,
  breakable,
  width = \linewidth,
  colframe = gray!25,
  colback  = gray!10,
  coltitle = black,  
  title    = \textbf{Box \thetcbcounter. #2},
  #1,
}

\usepackage[textsize=tiny]{todonotes}

\title{Can In-Context Learning Support Intrinsic Curiosity?}

\author{
    Eric Elmoznino${}^{*,1}$, 
    Sangnie Bhardwaj${}^{*,2}$,
    Johannes von Oswald${}^{1}$,
    Rajai Nasser${}^{1}$,\\
    \textbf{Blaise Agüera y Arcas${}^{1}$},
    \textbf{João Sacramento${}^{1}$},
    \textbf{Rif A. Saurous${}^{1}$},
    \textbf{Guillaume Lajoie${}^{1}$} \\\\
    ${}^1$Google -- Paradigms of Intelligence Team,
    ${}^2$Google DeepMind
    \vphantom{
    \thanks{Equal contribution. Correspondence to: \texttt{\{eric.elmoznino,guillaume.lajoie\}@gmail.com}, \texttt{sangnie@google.com}}
    }
}

\DeclareMathOperator{\Ent}{H}
\DeclareMathOperator{\pmi}{pmi}
\newcommand{\rone}{r^{\mathrm{one}}}
\newcommand{\rsum}{r^{\mathrm{sum}}}
\newcommand{\rsur}{r^{\mathrm{sur}}}
\newcommand{\rtask}{r^{\mathrm{task}}}
\newcommand{\rdl}{r^{\mathrm{dl}}}

\newcommand{\calF}{\mathcal{F}}
\newcommand{\rgap}{r^{\mathrm{gap}}}

\begin{document} 
\maketitle

\vspace{-2em}
\begin{abstract}
    Effective machine learning depends not only on how we model data, but also on what data we choose to collect. While large sequence models have revolutionized data modeling, the problem of automated data selection, or ``intrinsic curiosity'', remains a significant challenge. Classic approaches incentivize exploration by rewarding an agent based on its ``learning progress'', which measures how much a newly acquired observation improves a world model's predictive ability. However, evaluating these rewards traditionally requires expensive inner loops of gradient descent updates within each trajectory, rendering them computationally impractical at scale. In this work, we investigate whether the emergent in-context learning (ICL) capabilities of sequence models can eliminate this bottleneck by serving as immediate, update-free world models. Specifically, we evaluate whether an exploration policy can be trained to maximize learning progress, using solely the prediction errors and counterfactual context manipulations of an in-context learner. We first prove that in general Markov decision processes, this is in fact impossible in an unbiased way: the resulting intrinsic rewards either suffer from nuisance terms that bias their estimation of true learning progress, or they cannot be implemented using an in-context learner's prediction errors. Conversely, we prove a positive result for a broad subclass of non-temporal settings, encompassing active learning and Bayesian Experimental Design: here, ICL-derived rewards successfully bound and asymptotically converge to the true learning progress. We corroborate our theory with controlled experiments across continuous and symbolic environments, demonstrating that our ICL-driven framework successfully trains curious data-collection policies that explore optimally.
\end{abstract}

\section{Introduction}

How should an agent collect data in an unknown environment to ensure optimal exploration? One foundational approach dictates that the data should maximally improve the agent's model of its environment \citep{schmidhuber1991curious,lindley1956measure}. Such a model can subsequently support arbitrary downstream tasks, either through explicit planning or by having acquired robust representations. To this end, prior work has proposed utilizing ``intrinsic rewards'' to drive the collection of data, which are defined strictly as a function of the trajectory of actions and observations collected by a policy, and assume no extrinsic ``task''.

Given a Bayesian prior on the world, an optimal intrinsic reward is Bayesian information gain (BIG), which measures the expected bits gained about the true environment's dynamics with every action \citep{itti2009bayesian,lindley1956measure}. However, BIG is difficult to compute in practice because it requires explicitly parameterizing the environment's dynamics and performing intractable Bayesian inference. While tractable prediction-based objectives have been proposed as alternatives \citep[e.g.,][]{schmidhuber1991curious}, they face two major hurdles: (1) their theoretical relationship to BIG remains poorly understood and (2) computing them requires expensive gradient descent updates to a world model that have slow credit assignment.

Overcoming this second bottleneck requires fast and data-efficient learning mechanisms. In-context learning (ICL) has emerged as a highly effective paradigm for addressing this problem: sequence models can act as amortized predictors that \emph{implicitly} approximate the Bayesian posterior predictive in a single forward pass, bypassing the need for explicit inference. This capability is hypothesized to partly drive the success of large language models \citep{radford2019language, brown2020language}, and has seen great success with Prior-Fitted Networks (PFNs) \citep{nagler2023statistical,muller2022transformers} such as TabPFN \citep{hollmann2022tabpfn} which are pretrained on a large prior over datasets to perform amortized learning on a new dataset at inference time.

In this work, we use ICL for evaluating prediction-based intrinsic curiosity rewards. We examine the degree to which data collected by a policy improves a world model learned purely in-context, and reward the policy based on this improvement. This raises several fundamental questions: First, is this possible with ICL, and for which intrinsic rewards? Second, does the resulting reward approximate BIG, and if so, under what circumstances? To answer these questions, we make the following contributions:
\begin{enumerate}[leftmargin=*, topsep=0pt, itemsep=-0.3pt]
    \item By formalizing in-context learners as implicit Bayesian predictors, we prove novel mathematical relationships between prediction-based intrinsic rewards and BIG, which were previously unknown.
    \item We provide a negative result for general Markov Decision Processes (MDPs), proving that prediction-based intrinsic rewards are biased estimators of BIG. Further, while some of these rewards can be implemented using ICL, we show that others, such as classic learning progress \citep{schmidhuber1991curious}, cannot.
    \item In contrast, we establish a positive result for Bayesian Experimental Design (BED) settings, demonstrating that several intrinsic rewards can be approximated using ICL, and that it is possible to  asymptotically approximate BIG for long trajectories with a particular reward structure.
    \item We conduct experiments that corroborate our theoretical findings, demonstrating the practical viability of this ICL approach to curiosity.
\end{enumerate}

\vspace{-0.5em}
\section{Problem Setting and Notation}
\label{sec:setting}

\paragraph{Environment.}

We consider a Bayes-Adaptive MDP (BAMDP) with states $s_t$, actions $a_t$, horizon $T$, and a latent environment parameter $\theta \sim p(\theta)$ fixed throughout an episode. Conditional on $\theta$, the environment's dynamics are Markov and stationary: $s_{t+1} \sim p(\cdot \mid s_t, a_t, \theta)$.  Throughout, we write $h_t \coloneqq (s_1, a_1, \ldots, s_{t-1}, a_{t-1})$ for the trajectory history just before $s_t$.

\paragraph{Intrinsic curiosity objective.}

Our goal is to train a policy $\pi_\phi(a_t \mid h_t, s_t)$ under a ``meta-RL'' setting \citep{duan2016rl} in which episodes are rolled out from environments sampled from $p(\theta)$. Crucially, there is no notion of ``task'' or ``extrinsic'' reward. We instead consider ``intrinsic'' rewards that are functions of only the trajectory $(s_1, a_1, \ldots, s_T)$, and which are ideally maximized when the trajectory is highly informative about the latent parameters $\theta$ governing the environment's dynamics. By training $\pi_\phi$ across a broad distribution of environments sampled from $p(\theta)$, we aim for it to optimally explore a \emph{new} environment sampled from this distribution at inference time. In information theoretic terms, we consider an ``optimal'' intrinsic reward as one that maximizes the per-step mutual information between an observed state and the latent environment parameters, conditioned on the history. In the literature, this is called Bayesian information gain (BIG) or Bayesian surprise \citep{itti2009bayesian,lindley1956measure}:
\begin{align}
    \label{eq:big}
    \text{BIG} := I(s_t ; \theta \mid h_t) .
\end{align}

\paragraph{In-context learning prediction errors.}

We consider whether a sequence model $\rho$ can be used to approximate BIG, using only its prediction errors and manipulations of its context. We assume that $\rho$ has been pretrained offline to perform next-state predictions on trajectories of BAMDPs sampled from prior $p(\theta)$ with uniform action selection --- $\rho$ is then frozen when training the policy. We further suppose that $\rho$ has been trained with sufficient capacity and data, such that it has learned to emit the Bayesian posterior-predictive:
\begin{align}
    \rho(s_{t} \mid h_{t}) = \int p(s_{t} \mid s_{t-1}, a_{t-1}, \theta)\, p(\theta \mid h_{t})\, d\theta .
\end{align}
This perspective of in-context learners as amortized Bayes-optimal predictors has been formalized in substantial prior work \citep{graumoya2024learning,mikulik2020meta,xie2022an,binz2024meta,nagler2023statistical,muller2022transformers}, and it serves as a starting point for our investigations. In our setting, $\rho$ can be seen as a meta-learner that \emph{implicitly} fits environment parameters $\theta$ in-context and exposes their predictions for novel transitions, all within a single forward pass. The speed of in-context learners makes them particularly promising for the purposes of optimizing certain intrinsic curiosity objectives, which, as we will see next, require us to repeatedly assess the learning progress of a world model over the course of a trajectory.

\section{Prior Work on Intrinsic Curiosity}
\label{sec:prior_work}

We briefly describe approaches to intrinsic curiosity relevant to our work --- \citet{aubret2023information} provides a recent review. We discuss intrinsic rewards that are less directly related to our work in \cref{app:prior_work}. We assume $p(\cdot \mid \cdot, {\hat{\theta}_t})$ is a predictive world model that has been trained on subtrajectory $(a_{1:{t-1}}, s_{1:t})$.

\paragraph{Surprisal.}

A common heuristic for encouraging exploration is to maximize a world model's surprisal, seeking observations for which it has high prediction error. The surprisal reward is
$r_t \coloneqq -\log \ p(s_t \mid s_{t-1}, a_{t-1}, \hat{\theta}_{t-1})$.
This approach has been used at least as far back as \citet{schmidhuber1991possibility} and is simple but surprisingly effective \citep{burda2018largescale,levy2025worldllm,hester2012intrinsically,chentanez2004intrinsically}.
% The intuition behind surprisal-based intrinsic rewards goes back to the foundations of information theory \citep{shannon1948mathematical}, which equates the negative log-likelihood of an event with its information quantity.
However, surprisal fails to distinguish between two sources of information: epistemic (i.e., learnable) and aleatoric (i.e., noise). The latter can be detrimental: a policy seeking surprisal will learn to sit in front of a ``noisy TV'', even if these observations contain no new useful information about the environment \citep{schmidhuber1991curious}. Thus, while surprisal excels in deterministic environments, it fails in stochastic settings \citep{burda2018largescale}.

\vspace{-0.5em}
\paragraph{Learning progress.}

Another principled approach to intrinsic curiosity is to maximize ``learning progress'', introduced in \citet{schmidhuber1991curious} and further explored in \citep{schmidhuber2009driven,schmidhuber2010formal,storck1995reinforcement,oudeyer2007intrinsic,oudeyer2008can,oudeyer2007what,lopes2012exploration,azar2019world}. It aims to reward a policy based on the degree to which the data that it acquires improves the predictive ability of a world model. The learning progress intrinsic reward is
$r_t \coloneqq \log \ p(s_t \mid s_{t-1}, a_{t-1}, \hat{\theta}_t) - \log \ p(s_t \mid s_{t-1}, a_{t-1}, \hat{\theta}_{t-1})$,
i.e. the improvement in the model's ability to predict $s_t$ \emph{after} the transition to it has been observed. $p_{\hat{\theta}_t}$ is often parameterized using a recurrent neural network updated online with gradient descent. Unlike surprisal, it assigns zero reward to observations that only contain aleatoric noise, since they do not improve the model's predictive ability.

In \cref{sec:theory}, we will show that while learning progress has desirable exploration properties, the need to evaluate a model's prediction error on transitions \emph{that it has already seen} makes it non-implementable using an in-context learner's prediction errors, for general environments.
\citet{azar2019world} introduced an alternative that instead evaluates improvements in prediction error on \emph{future} transitions $(s_{t+K - 1}, a_{t+K-1}, s_{t+K})$, rather than on the current one. \cref{sec:theory} will show that this alternative --- along with a variant that considers improvements on \emph{all} future transitions --- \emph{can} be computed using an in-context learner's prediction errors.

\vspace{-0.5em}
\paragraph{Bayesian information gain.}

Alternatively, we can seek data that maximally changes a model's \emph{belief} over possible environments --- an approach termed Bayesian information gain (BIG) or Bayesian surprise \citep{sun2011planning,little2013learning,itti2009bayesian,houthooft2016vime,stadie2015incentivizing,mackay1992information,fedorov1972theory,houlsby2011Bayesian,lindley1956measure}. The predictive model now considers a distribution over possible environment dynamics, with initial prior $p(\hat{\theta})$. \citet{sun2011planning} show that the optimal exploration strategy is to maximize the Kullback-Leibler divergence between the prior and posterior over environment parameters after having observed a novel transition:
$r_t \coloneqq KL(p(\hat{\theta} \mid h_t, s_t) \;||\; p(\hat{\theta} \mid h_t))$,
where $p(\hat{\theta} \mid h_t, s_t) \propto p(\hat{\theta} | h_t)p(s_t | h_t, \hat{\theta})$. Intuitively, this objective encourages the agent to reduce its uncertainty about the environment dynamics as quickly as possible. Furthermore, if the initial prior $p(\hat{\theta})$ is equal to the environment's true prior distribution over latents $p(\theta)$, this reward is mathematically equivalent in expectation to $I(s_t ; \theta \mid h_t)$ from \cref{eq:big}.

While optimal in terms of yielding the fastest convergence of $p(\hat{\theta} \mid h_t, s_t)$ towards the true environment dynamics $\theta$, BIG poses significant implementation challenges. It requires a model to do Bayesian inference over a space of possible environments --- an operation that is generally intractable and difficult to approximate. Moreover, this approach assumes the model can \emph{explicitly} parameterize the hypothesis space over $\theta$, which is challenging in rich environments with unknown latent structure. In \cref{sec:theory}, we will take BIG as the theoretically-optimal objective for intrinsic curiosity, and we will evaluate the degree to which more tractable prediction-based rewards computed with in-context learners can approximate it.

\section{Computing Intrinsic Rewards Using In-Context Learners}
\label{sec:theory}

We now ask whether an in-context learner $\rho$ can support intrinsic curiosity rewards that approximate BIG. \Cref{sec:predrewards} writes candidate rewards in $\rho$'s interface; \cref{sec:limits} shows none can identify BIG in general BAMDPs at any finite horizon, and that the asymptotic limits required to escape this barrier are practically out of reach; \cref{sec:BED} restores tractability under a structural restriction, namely that of Bayesian Experimental Design (BED).

\subsection{Predictive Rewards}
\label{sec:predrewards}

We retain the BAMDP setup of \cref{sec:setting} and additionally write $h_{t' \setminus t}$ for the history $h_{t'}$ with $s_t$ replaced by a ``mask'' token --- a counterfactual manipulation of $\rho$'s context used by learning-progress-style rewards.

We assume:
(i) Markovian, stationary dynamics given $\theta$,
(ii) $\rho$ is exactly Bayesian over $\theta$,
(iii) the masked trajectory $h_{t' \setminus t}$ removes only $s_t$'s contribution and $\rho$ remains exactly Bayesian on it,
(iv) posterior consistency: $p(\theta \mid h_t) \to \delta_{\theta^*}$ as $t \to \infty$,
(v) actions are policy-generated; future actions do not update the posterior over $\theta$,
(vi) mixing given $\theta$: $I(s_t; s_{t+k} \mid \theta, h_t, a_{t:t+k-1}) \to 0$ as $k \to \infty$.
Detailed treatments, including the consequences of approximation, masked-input training, and non-mixing dynamics, are in \cref{app:assumptions}.

We want $\pi_\phi$ to collect observations that are maximally informative about the world, and we investigate how to maximize the stepwise Bayesian information gain (BIG) contributed by a new observation $s_t$: $I(s_t; \theta \mid h_t)$. Importantly, we do not allow explicit access or manipulation of $\theta$, and consider intrinsic rewards reviewed in \cref{sec:prior_work} that are built off of state prediction errors, expressed in terms of $\rho$ when possible. Our theoretical results apply to a general class of rewards, including the following common ones:
\begin{align}
  \rsur_t &\coloneqq -\log \ \rho(s_t \mid h_t) , \label{eq:rsur} \\
  \rdl_t &\coloneqq \log \int p(s_t \mid s_{t-1}, a_{t-1}, \theta)\, p(\theta \mid h_t, s_t)\, d\theta - \log \ \rho(s_t \mid h_t). \label{eq:rdl}
\end{align}
Here, $\rsur$ is classical predictive surprisal. $\rdl$ is a learning progress reward based on description length (dl) reduction \citep{schmidhuber1991curious,schmidhuber2009driven,schmidhuber2010formal}, which asks how much a just-observed transition would improve a model's predictions if included in its posterior. Adding to this list, we propose a novel reward:
\begin{align}
    \rsum_t &\coloneqq \sum_{t' = t+1}^{T} \bigl[\log \ \rho(s_{t'} \mid h_{t'}) - \log \ \rho(s_{t'} \mid h_{t' \setminus t})\bigr] , \label{eq:rsum}
\end{align}
where $T$ is the trajectory length. $\rsum$ is one of our core contributions, and experiments show it performs well (\cref{sec:experiments}). It telescopes over the remaining trajectory and is an extension of the NDIGO reward of \citet{azar2019world}, which measured the difference in predictive log-likelihoods for a single future observation $K$ steps in the future. We refer the reader to \cref{app:rone} for a treatment of the NDIGO reward.
We illustrate our setting in \cref{fig:method}, using $\rsum$ as an example.
\begin{figure}[ht]
    \centering
    \includegraphics[width=\textwidth]{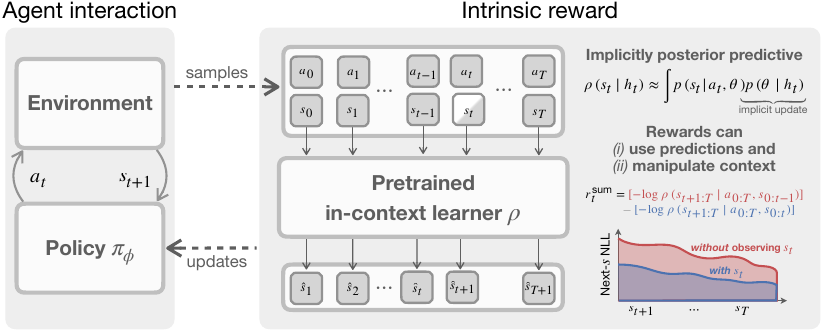}
    \caption{Our method involves using a pretrained in-context learner $\rho$ to construct intrinsic curiosity rewards for a policy $\pi_\phi$. Trajectories unrolled by the policy are passed to the in-context learner, and the reward can be any function of the resulting observation prediction errors on manipulated sequence contexts. We give an example for the reward $\rsum$, which measures the improvement in future prediction errors when a particular state is observed compared to when it is masked.}
    \label{fig:method}
\end{figure}

\subsection{Limits on Intrinsic Rewards From In-Context Learners}
\label{sec:limits}

We first characterize the class of rewards easily implementable with in-context learners $\rho$:
\begin{definition}[Class $\calF$]
\label{def:F}
$\calF$ is the class of arbitrary functions of finitely many predictive likelihoods of $\rho$ at conditioning subsets of the actual trajectory:
\begin{align*}
    \calF \;=\; \Bigl\{ r_t = f\bigl((\rho(X_i \mid Y_i))_{i \in I}\bigr) :
    X_i \text{ observation, } Y_i \subseteq \text{trajectory, }\forall i\in I\,,\text{ and } f:[0,1]^{|I|}\to \mathbb{R} \Bigr\}.
\end{align*}
\end{definition}
At any finite horizon $T$, $\rsur$ and $\rsum$ belong to $\calF$. $\rdl$ does not: the integral in \cref{eq:rdl} requires an explicit Bayesian update of $\rho$'s posterior --- outside $\rho$'s predictive interface in general (see \cref{app:rdl-in-bamdbs-vs-beds}). We return to $\rdl$ in \cref{sec:BED}, where it \emph{does} admit a $\rho$-tractable form for a subclass of environments. Returning to our goal of maximizing BIG, no finite-horizon reward in $\calF$ can be an unbiased estimator:

\begin{restatable}[Impossibility of BIG identification in $\calF$]{theorem}{ThmImpossibility}
\label{thm:impossibility}
Let $\mathcal{M}$ be the class of BAMDPs satisfying (i)--(v) above, whose transition kernels are neither deterministic nor independent across time given $\theta$. Then for every $r_t \in \calF$, there exists $M \in \mathcal{M}$ with $\mathbb{E}^M[r_t \mid h_t] \ne I(s_t; \theta \mid h_t)$.
\end{restatable}
The proof (\cref{app:proof-impossibility}) constructs two BAMDPs whose priors share enough leading moments to match $\mathbb{E}[r_t]$ for any \emph{finite} $r_t \in \calF$ while their expected BIG differs. We now decompose $\rsur$ and $\rsum$ against $I(s_t; \theta \mid h_t)$ to locate the bias.
\begin{restatable}[Decomposition of $\rsur$]{theorem}{ThmRsur}
\label{thm:rsur}
Under (i)--(iii),
\begin{align*}
    \mathbb{E}[\rsur_t \mid h_t]
    \;=\; {I(s_t; \theta \mid h_t)}
    \;+\; \underbrace{\Ent(s_t \mid h_t, \theta)}_{\textnormal{aleatoric entropy}} .
\end{align*}
$\Ent(s_t \mid h_t, \theta) = \Ent(s_t \mid s_{t-1}, a_{t-1}, \theta)$ is non-negative and does not vanish as $\theta$ becomes identified.
\end{restatable}
The proof is in \cref{app:proof-rsur}. \Cref{thm:rsur} highlights the classic ``noisy TV problem'', where an agent seeks sources of unpredictable noise.
\begin{restatable}[Decomposition of $\rsum$]{theorem}{ThmRsum}
\label{thm:rsum}
Under (i)--(iii) and (v),
\begin{align*}
    \mathbb{E}[\rsum_t \mid h_t, a_{t:T-1}]
    \;=\; I(s_t; \theta \mid h_t)
    \;+\; \underbrace{I(s_t; s_{t+1} \mid \theta, h_t, a_t)}_{\textnormal{one-step ``abductive''}}
    \;-\; \underbrace{I(s_t; \theta \mid h_t, a_{t:T-1}, s_{t+1:T})}_{\textnormal{``residual''}}.
\end{align*}
Both terms are non-negative. Under (iv), the residual vanishes as $T \to \infty$; the abductive does not.
\end{restatable}
The proof is in \cref{app:proof-rsum}. The \emph{abductive} term captures kernel-mediated coupling between $s_t$ and the next observation $s_{t+1}$ that is not explained by their common dependence on $\theta$. Similarly, a general treatment of  rewards involving log-ratios lead to signal/abductive/residual decomposition (\cref{app:log-ratio}). Notably, applied to NDIGO, it produces an analogous decomposition whose abductive and residual both fail to vanish.

\paragraph{Beyond \texorpdfstring{$\calF$}{F}: infinite-horizon rewards.}

Taking $T \to \infty$ to drive $\rsum$'s residual to zero already pushes $\rsum$ outside the finite-horizon class $\calF$ of \cref{thm:impossibility}. Are there infinite-horizon rewards built from $\rho$'s predictives that recover BIG in general BAMDPs? Consider the generalization of $\rsum$ with a gap parameter $K$ separating $s_t$ from the predicted block, $\rgap_t(K) \coloneqq \log \ \rho(s_{t+K:T} \mid h_t, s_t, a_{t:T-1}) - \log \ \rho(s_{t+K:T} \mid h_t, a_{t:T-1})$. Under the mixing assumption (vi), $\rgap$ recovers BIG in the iterated limit $\lim_{K \to \infty} \lim_{T \to \infty}$ (\cref{cor:rgap-BIG}), and this double limit is structurally necessary --- any log-ratio reward universally identifying BIG must place all weight at infinite gap with blocks of unbounded size (\cref{thm:double-limit-necessity}). We use $\rgap$ as an analytical tool only: it is not feasibly tractable using in-context learners (see below; details in \cref{app:gap-reward}).

\paragraph{Summary.}

The section identified compounding obstructions to building $\rho$-based estimators of BIG:
\begin{itemize}[leftmargin=*, topsep=0pt, itemsep=0pt]
    \item \emph{Interventional rewards} --- those requiring modifications of $\rho$'s posterior --- lie outside $\rho$'s standard interface. The canonical example, $\rdl$ (\cref{eq:rdl}), folds the just-observed $s_t$ into the posterior over $\theta$ before re-evaluating the kernel at $(s_{t-1}, a_{t-1})$, and is not implementable through $\rho$ in general BAMDPs.
    \item Among $\rho$-implementable rewards in $\calF$, no finite-horizon estimator identifies BIG (\cref{thm:impossibility}); $\rsur$ and $\rsum$ exemplify the bias structure, the former acquiring an aleatoric ``noisy TV'' offset (\cref{thm:rsur}) and the latter a persistent one-step abductive (\cref{thm:rsum}).
    \item Recovery of BIG via a reward related to $\calF$ requires both sequence length $T \to \infty$ and marginalizing gap $K \to \infty$, as exemplified by the theoretical reward $\rgap$ (\cref{cor:rgap-BIG,thm:double-limit-necessity}). The latter is a substantial obstacle: significant unobserved intermediate states $s_{t+1:t+K-1}$ that $\rho$ must marginalize over are poorly supported in sequence models (\cref{app:masking_training}).
\end{itemize}
In sum, we argue there is no practical reward easily implementable with an in-context learner $\rho$ that converges to BIG in general BAMDPs. Moreover, $\rsum$ stands as the best $\rho$-implementable approximation of BIG --- scaling naturally with trajectory length but nevertheless retaining an irreducible abductive bias. We now consider BED, in which several of these obstructions vanish.

\subsection{Bayesian Experimental Design Setting}
\label{sec:BED}

We now consider \emph{Bayesian Experimental Design} (BED) environments, where the transition kernel satisfies $p(s_t \mid s_{t-1}, a_{t-1}, \theta) = p(s_t \mid a_{t-1}, \theta)$ --- the effect of action $a_t$ depends on $\theta$ but not on the current state. This covers any sequential experiment with trials conditionally independent given $\theta$ and the action sequence, such as active learning. Aleatoric biases remain, and $\rsur$'s decomposition is unchanged (\Cref{thm:rsur}). But, any abductive bias of the form seen in \cref{thm:rsum} vanishes structurally, simplifying $\rsum$:
\begin{restatable}[Decomposition of $\rsum$ in BED]{corollary}{CorRsumBED}
\label{cor:rsum-BED}
Under (i)--(v),
\begin{align*}
    \mathbb{E}[\rsum_t \mid h_t, a_{t:T-1}]
    \;=\; I(s_t; \theta \mid h_t)
    \;-\; I(s_t; \theta \mid h_t, a_{t:T-1}, s_{t+1:T}) ,
\end{align*}
i.e., the abductive of \cref{thm:rsum} vanishes structurally. As $T \to \infty$, $\mathbb{E}[\rsum_t] \to I(s_t; \theta \mid h_t)$.
\end{restatable}
The proof is in \cref{app:proof-rsum-BED}. Thus $\rsum$ is an asymptotically unbiased BIG estimator in BED.
For $\rdl$ in the BED setting, things also improve. The posterior predictive term admits a $\rho$-predictive form via a counterfactual action commitment that ``copies'' the current transition $(s_{t-1}, a_{t-1}, s_t)$ into $\rho$'s context:
\begin{equation}
\label{eq:rdl-counterfactual}
  \int p(s_t \mid s_{t-1}, a_{t-1}, \theta)\, p(\theta \mid h_t, s_t)\, d\theta
  \;=\; \rho\bigl(s_{t+1} = s_t \,\big|\, h_t, s_t,\, a_t = a_{t-1}\bigr) .
\end{equation}
The counterfactual is the action choice $a_t = a_{t-1}$ (re-using the same action), valid because actions are policy-generated and the BED environment's transition kernel has no $s_{t-1}$ dependence. The identity in \cref{eq:rdl-counterfactual} makes $\rdl$ implementable using in-context learners at the cost of a hypothetical extra rollout step.
\begin{restatable}[Decomposition of $\rdl$ in BED]{theorem}{ThmRdlBED}
\label{thm:rdl-BED}
Under (i)--(v), with
$L(\theta) \coloneqq p(s_t \mid a_{t-1}, \theta)/\rho(s_t \mid h_t)$,
\begin{align*}
    \mathbb{E}[\rdl_t \mid h_t]
    \;=\; I(s_t; \theta \mid h_t)
    \;+\; \mathbb{E}_{s_t \mid a_{t-1}}\!\Bigl[\log \ \mathbb{E}_{p(\theta \mid h_t, s_t)}[L(\theta)]
                                              - \mathbb{E}_{p(\theta \mid h_t, s_t)}[\log  \ L(\theta)]\Bigr] ,
\end{align*}
i.e., BIG plus a non-negative Jensen gap. Both terms in the expectation vanish as $t \to \infty$ under (iv).
\end{restatable}
The proof is in \cref{app:proof-rdl-BED}.
Like $\rsum$, $\rdl$ converges to BIG in BED, but with opposite-sign bias: $\rsum$ is biased downward by a residual on future data, $\rdl$ upward by a Jensen gap on past data. We thus get $\mathbb{E}[\rsum_t] \le I(s_t; \theta \mid h_t) \le \mathbb{E}[\rdl_t]$, with the gap closing asymptotically. Importantly, however, it should be noted that $\rdl$'s bias (the Jensen gap) vanishes together with BIG $I(s_t; \theta \mid h_t)$ itself, as $t\to\infty$ (see \cref{app:proof-rsum-BED}). 

\vspace{-0.5em}
\paragraph{Summary.}

BED removes the obstructions of general BAMDPs: $\rsum$ and $\rdl$ both asymptotically recover BIG. However, $\rsum$ recovers BIG as $T\to \infty$ while $\rdl$ does so as $t\to\infty$. This means that given long sequences, $\rsum$ can, in principle, recover BIG at finite $t$ while $\rdl$ only trivially recovers BIG in the limit where the signal itself vanishes. Nevertheless, both rewards can still be useful at finite $T$ and finite $t$, as we now investigate in experiments.

\vspace{-0.5em}
\section{Experiments}
\label{sec:experiments}
\vspace{-0.5em}

We evaluate our framework on three structured BED environments: Gaussian Process function estimation, Mastermind code-breaking, and Alchemy transition-rule discovery, each requiring the policy to actively gather observations that are informative about unknown latent variables. 

We train policies with PPO or REINFORCE using curiosity-driven rewards from \cref{sec:predrewards}: $\rsur$, $\rdl$, and $\rsum$. We also compare to $\rtask$, which uses a validation metric of the environment (indicative of having learned its dynamics) as the reward.
For each environment, we pretrain an in-context learner $\rho$ by sampling environment latents from the prior $\theta \sim p(\theta)$ and unrolling trajectories with uniformly-random actions; as such, we call these Prior-Fitted Networks (PFNs) in line with prior work \citep{muller2022transformers}. Some environments additionally admit an exact predictive Bayesian oracle predictive $\rho^*$. 
% For each environment we additionally consider a \emph{noisy} variant in which a corruption mechanism distorts some observations; in these settings the Bayesian oracle is intractable, whereas the trained $\rho$ can be applied in both settings. 
Details and results for each environment are given below.

\vspace{-0.5em}
\subsection{Gaussian Process}

\begin{figure}[ht]
    \centering
    \includegraphics[width=\textwidth]{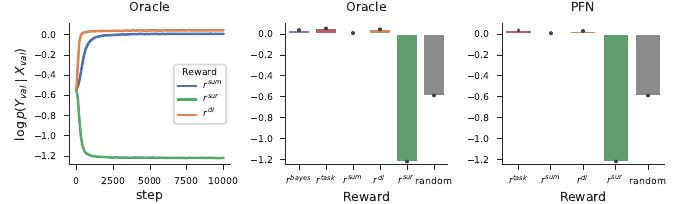}
    \caption{Comparison of validation log-likelihood for Gaussian Processes for the Oracle and PFN predictive models. (\emph{Left}) For $r^{sur}$ validation score decreases over training steps. (\emph{Center and right}) Final validation scores: $r^{sum}$ and $r^{dl}$ achieve equivalent performance to training on $\rtask$, but $r^{sur}$ performs significantly worse than a random policy.}
    \label{fig:gp_bar}
\end{figure}

The Gaussian Process (GP) environment models active function estimation in a continuous domain.
An unknown function $f$ is sampled from a GP prior with a rational quadratic kernel. The function $f$ is represented by a finite set of inducing points sampled on a regular grid of resolution $R$ over the domain
$[-x_{\max}, x_{\max}]^2$;
the posterior mean at any query location is obtained by exact GP conditioning on these inducing points. More details can be found in \cref{app:gp}.
 
A fresh function $f$ is drawn at the start of each episode. At each step the policy selects a query location $x_t \in [-x_{\max}, x_{\max}]^2$ and observes a noisy evaluation $y_t = f(x_t) + \varepsilon_t$, where  $\varepsilon_t \sim \mathcal{N}\bigl(0,\, \sigma_{\mathrm{noise}}^2(x_t)\bigr)$
is spatially varying noise. The noise is in the form of a tiled checkerboard pattern (details in \cref{app:gp}), and is the same in every episode. An optimal information-gathering policy should therefore learn to preferentially query the low-noise tiles. Consecutive observations are conditionally independent given $f$ and the noise map. Therefore, the environment follows the BED setting with $\theta=f$, $a_t = x_t$, and $s_{t+1} = y_t$.  

\paragraph{Predictive model.}

The Bayesian oracle $\rho^*$ performs exact GP posterior inference.
After observing $h_t$, the predictive distribution for $y_t$ is  
$\rho^*(y_t \mid h_t, x_t) = \mathcal{N}\!\bigl(y_t;\; \mu_t(x_t),\; \sigma^2_t(x_t)\bigr)$
where $\mu_t$ and $\sigma^2_t$ are the standard GP
posterior mean and predictive variance (including observation
noise) conditioned on $h_t$.
Note $\rho^*$ has access to the true kernel and the noise map. The learned PFN $\rho$ is a causal Transformer. Details about the training and architecture are in \cref{app:gp-worldmodel}. We also show in \cref{fig:gp-corr} that $\rho$ approximates $\rho^*$ well.

\begin{figure}[ht]
    \centering
    \includegraphics[width=0.37\textwidth]{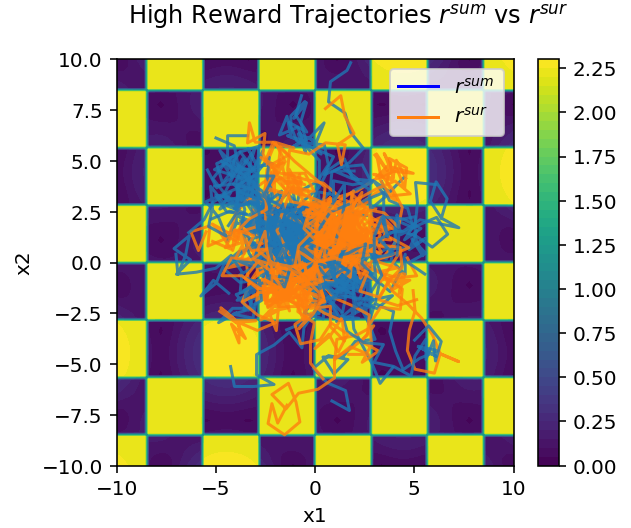}
    \includegraphics[width=0.61\textwidth]{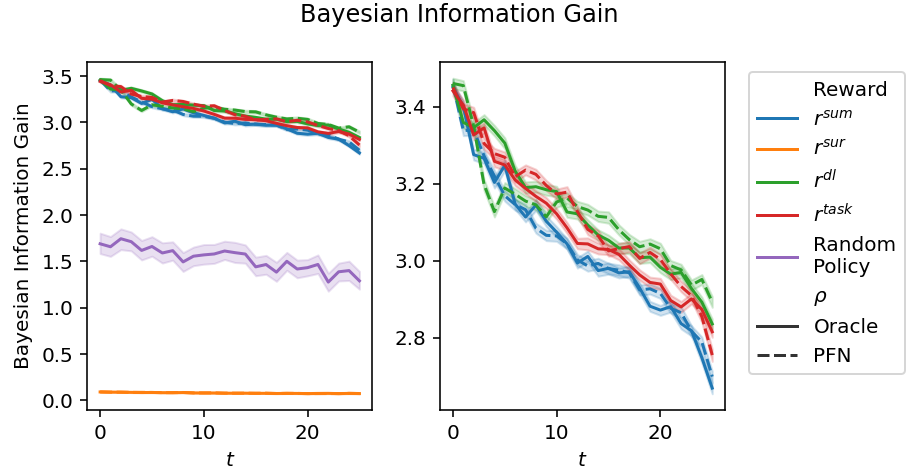}
    \caption{\emph{Left:} The colourbar depicts noise strength of the GP. Top 10\% of random paths ranked by the rewards show that $\rsum$ is high for the noiseless tiles, conversely $\rsur$ favours the noisy regions. \emph{Center:} BIG across trajectories generated by trained policies. 
    $r^{sur}$ has the lowest information gain due to the Noisy TV problem. 
    \emph{Right:} Zoomed comparison.
    Higher BIG for $\rsum$, $\rdl$, and $\rtask$ indicates that these policies select queries that are more informative about the underlying function. 
    % Higher BIG indicates the policy selects queries that are more informative about the underlying function. \emph{Right:} Zoomed comparison. $\rsum$, $\rdl$, and $\rtask$ are closely grouped, whereas $r^{sur}$ has the lowest information gain due to the Noisy TV problem.
    }
    \vspace{-10pt}
    \label{fig:gp-all}
\end{figure}

\paragraph{Validation.}

To measure the quality of the collected trajectory $h_T$, we evaluate on a set of validation points on the grid. Let $(X_{val}, Y_{val}) = \{(\tilde x_j, \tilde y_j)\}_{j=1}^{M}$ be a random subset of the episode's ground-truth GP inducing points.
Regardless of whether $\rho$ is a Bayesian oracle or a trained PFN, we validate by conditioning the oracle $\rho^*$ on the trajectory $h_T$ and report the mean log-likelihood
$V = \frac{1}{M}\sum_{j=1}^{M} \log\rho^*(\tilde y_j \mid h_T,\,\tilde x_j)$.

\paragraph{Results.}

Since the environment contains aleatoric noise, we observe a stark divergence in performance between the rewards. As shown in \cref{fig:gp_bar}, both the random baseline and the policy trained with $\rsur$ perform significantly worse than those trained with $\rdl$ and $\rsum$. We note that the validation log-likelihood steadily decreases over training, despite $\rsur$ increasing (\cref{fig:gp-rsur-train}). 
This provides strong empirical confirmation of \cref{thm:rsur}, which states that $\rsur$ is susceptible to the ``noisy TV'' problem, whereas $\rsum$ and $\rdl$ successfully distinguish between reducible epistemic uncertainty and irreducible aleatoric noise. 

To further confirm this, in \cref{fig:gp-all}, we collect 1000 trajectories with a random policy. We then calculate $\rsur$ and $\rsum$ for these trajectories, and in green highlight the ones in the 90\% percentile for each reward. We can observe that the trajectories with the highest $\rsum$ stray away from the noisy checkerboards, whereas those for $\rsur$ concentrate in these regions instead. We also plot the BIG across trajectories collected by the policies trained with the different rewards. As predicted, the BIG of $\rsur$ is much lower than that of other rewards. Finally, in \cref{app:decomposition}, we show empirically that $\rsum$ and $\rdl$ decompose into BIG and a shrinking finite-time nuisance term, as predicted by \cref{cor:rsum-BED} and \cref{thm:rdl-BED}.

\subsection{Mastermind}

\begin{figure}[ht]
    \centering
    \includegraphics[width=\textwidth]{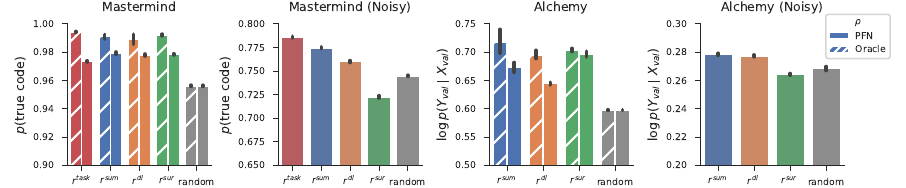}
    \caption{Validation scores for Mastermind and Alchemy. In the standard variant, all rewards are competitive. In both the noisy variants, $r^{sur}$'s performance falls below the random baseline, whereas $r^{sum}$ and $r^{dl}$ remain robust.}
    \label{fig:mm-alchemy}
\end{figure}

Mastermind is a code-breaking game parameterized by a secret code $c \in \{0,\dots,C{-}1\}^{L}$ made of coloured pegs, where $L$ is the code length and $C$ the number of colours. At each step, the policy submits a guess $g_t$ and receives feedback $(b_t, w_t)$, where $b_t $ counts the number of pegs correct in both colour and position, and $w_t$ counts the total number of pegs correct in colour irrespective of position. 
The feedback depends only on the current guess and the fixed code. This follows the BED setting with $\theta = c$, $a_t = g_t$, and $s_{t+1} = (b_t, w_t)$.

We additionally consider a noisy variant of the environment in which one randomly chosen colour per episode is designated as ``evil.'' Any guess position whose colour matches the evil colour has that position's value replaced with a uniformly random colour before feedback is computed, so the resulting $(b_t, w_t)$ may not reflect the true code. Additionally, the evil colour is guaranteed not to appear in the secret code itself. The evil colour identity is unknown to the policy and the learned world model $\rho$. 

\vspace{-0.5em}
\paragraph{Predictive model.}

The oracle $\rho^*$ maintains an exact posterior over secret codes by enumerating all $C^L$ possibilities and filtering to those consistent with the observed history. It computes the predictive probability of each new response as the fraction of consistent codes that would produce that response:
\[
  \rho^*(s_t \mid h_t)
  = \frac{|\{c' : c' \text{ consistent with } h_t \text{ and } (a_{t-1}, s_t)\}|}
         {|\{c' : c' \text{ consistent with } h_t\}|}.
\]
$\rho^*$ is only tractable when the environment is deterministic.
The learned PFN $\rho$ is a causal Transformer pretrained on randomly generated trajectories. More details about consistency, the PFN architecture, and training process are in \cref{app:mm-bayes,app:mm-worldmodel}.

\vspace{-0.5em}
\paragraph{Validation.}

After the policy collects a trajectory of $T$ guesses and responses $h_T$, we measure how much probability mass the world model places on the true secret code.
We construct a synthetic action/observation pair $\tilde{a}_v = c$, $\tilde{s}_v = (L,L)$ --- the true code paired with a perfect-score response --- append it to the trajectory, and report $V = \rho^*(\tilde{s}_v \mid h_T, \tilde{a}_v)$. For the noisy variant, we use $\rho$ instead of $\rho^*$.

\vspace{-0.5em}
\paragraph{Results.} 

As seen in \cref{fig:mm-alchemy}, $\rsur$ achieves near-optimal performance in the standard setup given that the environment has no source of aleatoric noise, matching the performance of training directly for the task. $\rsum$ and $\rdl$ also achieve comparable performance. In this setting, all reward functions perform consistently across both the Bayesian oracle and the trained PFN, further validating that sequence models can serve as effective predictive models for intrinsic reward computation.

However, the introduction of noisy pegs significantly alters these dynamics. In the noisy variant, the performance of the policy trained with $\rsur$ falls even below the random baseline. Conversely, policies trained with $\rsum$ and $\rdl$ remain robust, retaining high performance despite the corruption mechanism. This confirms that $\rsur$ becomes a liability in environments with stochastic transitions, while $\rsum$ and $\rdl$ consistently identify informative actions.

\subsection{Alchemy}

Alchemy \citep{wang2021alchemy} is a symbolic meta-learning environment. Each environment defines a fixed set of transition rules mapping (stone, potion) input pairs to transformed output stones, and the agent must discover these rules through sequential experimentation.

The environment is parameterized by a discrete ID $\theta$, which determines the transition dynamics. 
Stones are four-dimensional discrete vectors whose dimensions have
$(3, 3, 3, 4)$ possible values respectively, and potions take one of 6 values.
At each step the policy selects both the initial stone and the potion $(i_t, p_t)$ and observes final stone $f_t$ from the environment.
Consecutive observations are conditionally independent given $\theta$
(i.e.\ $i_{t+1}$ is chosen freely and need not equal $f_t$).
Therefore, Alchemy is BED with $a_t = (i_t, p_t)$, and $s_{t+1} = f_t$.

Similar to Mastermind, we consider a noisy variant of the environment in which one randomly chosen stone type per episode is designated as ``evil.'' Any transition whose input stone matches the evil stone has its output replaced with a uniformly random stone.

\vspace{-0.5em}
\paragraph{Predictive model.}

% The oracle $\rho^*$ performs exact Bayesian inference over environment IDs.
Like Mastermind, the oracle $\rho^*$ performs exact Bayesian inference by enumerating all $N$ candidates, and eliminating those that are not consistent (all observed transitions match $\theta$'s rules). $\rho^*$ is only tractable when the environment is deterministic.
%The predictive probability of a new observation $(a_t, s_t)$ is
%\[
%  \rho^*((a_t, s_t) \mid h_t)
%  = \frac{|\{e' : e' \text{ consistent with } o_{1:t}\}|}
%         {|\{e' : e' \text{ consistent with } o_{1:t-1}\}|}.
%\]
The learned PFN $\rho$ is a causal Transformer with an
autoregressive GRU decoder that factorizes output prediction across the
components of $f_t$ (see \cref{app:alchemy_architecture} for full
architectural and training details).

\vspace{-0.5em}
\paragraph{Validation.}

After the policy collects a trajectory of $T$ transitions $h_T$, we
evaluate how well the world model can predict held-out transitions from
the same environment.
From the ground-truth transition table of the current environment ID,
we randomly sample $M{=}10$ transitions, excluding any whose initial
stone matches the evil stone (so that only clean transitions are used).
These are appended sequentially to the trajectory, and we report the
mean predictive log-probability
$V = \frac{1}{M}\sum_{j=1}^{M}
     \log\rho(\tilde{s}_j \mid h_{1:T},\,\tilde a_{1:j-1}, \tilde{s}_{1:j-1})$.
For the non-noisy environment we evaluate under both the Bayesian oracle
$\rho^*$ and the trained PFN $\rho$; for the noisy (evil stone) variant we evaluate under $\rho$ only.

\vspace{-0.5em}
\paragraph{Results.}

We observe a similar pattern as Mastermind in the Alchemy environment in \cref{fig:mm-alchemy}. In the noiseless task, $\rsur$ and $\rsum$ both exhibit strong performance, followed by $\rdl$. Under the noisy variant, $\rsur$ again demonstrates a failure to generalize, with performance below that of a random policy. In contrast, $\rsum$ maintains its superiority with both oracle and PFN, achieving the highest validation log-probability. 

The results across all three domains underscore the limitation of surprisal-based curiosity in realistic noisy settings, and demonstrate the robustness of $\rsum$ and $\rdl$. They also demonstrate the effectiveness of trained in-context learners as predictive models in these settings, as replacements for the intractable BIG.

\section{Discussion and Future Work}
\label{sec:discussion}

We demonstrated that in-context learners can efficiently evaluate prediction-based intrinsic curiosity rewards, under certain conditions. In Bayesian Experimental Design (BED) settings, our proposed reward $\rsum$ and description length reward from prior work $\rdl$ asymptotically approximate Bayesian information gain (BIG) and drive exploration effectively, though $\rdl$ suffers from asymptotic pathologies that $\rsum$ avoids. Both avoid the pitfalls of standard surprisal. This framework opens several promising directions for future work.

\vspace{-0.5em}
\paragraph{General BAMDPs.}

Extending this approach to general Bayes-Adaptive Markov Decision Processes (BAMDPs) requires handling environments with temporal structure. Implementing $\rdl$ or BIG requires access to a model's parameters, which is challenging using in-context learners that only expose the model's predictions. A promising approach is to project the in-context learner's hidden states into ``task vectors'' \citep{hendel2023context} that represent the environment's latent dynamics. Modulating or manipulating these task vectors counterfactually could allow the direct estimation of $\rdl$ or BIG in latent space. Attempting to approximate $\rgap$ --- which asymptotically converges to BIG --- using long context lengths and gaps is also a possibility.

\vspace{-0.5em}
\paragraph{Training the in-context learner.}

Matching the true Bayesian posterior predictive over the distribution of policy-sampled trajectories during training poses significant challenges. Large foundation models trained on broad data mixtures offer powerful in-context learning, but they are often poorly calibrated to the true Bayesian posterior predictive \citep{badola2025multi}. Furthermore, standard causal Transformers suffer from length generalization issues and positional encoding leakage, breaking the martingale properties necessary for Bayesian coherence \citep{falck2024context}. While this is partially mitigated in BED by using permutation-invariant sequence models trained on synthetic priors, covariate shift remains a primary failure mode. Specifically, a model trained on trajectories from a random or fixed policy will degrade when evaluated on informative trajectories collected by an active, reward-maximizing deployment policy \citep{yadlowsky2023pretraining}. Addressing this distribution shift may require techniques that jointly fine-tune the predictive model and the policy \citep{ross2011reduction,ivanova2024step,azar2019world}, though optimizing against non-stationary intrinsic rewards from a changing $\rho$ might introduce instability.

\vspace{-0.5em}
\paragraph{Online RL setting.}

This work focuses on an episodic meta-RL setting, but translating these methods to single-lifetime online reinforcement learning introduces several difficulties. First, the fixed context window of sequence models limits the duration over which an agent can accumulate learning progress. Second, in-context learning empirically exhibits saturation; predictive improvements decrease as the context length grows \citep{elmoznino2025incontext}, perhaps requiring periodic weight updates to consolidate long-term knowledge \citep{bornschein2024Transformers}. Third, the rewards themselves become computationally intractable in continuous online settings. Currently, the intrinsic reward computation is naturally episodic and calculated only at the conclusion of a trajectory (where $\rsum$ evaluates over the remainder of the fully unrolled sequence, and $\rdl$ requires trajectory reruns with the new parameters). Approximating these episodic metrics for online scaling will require stochastic sampling or truncated evaluation windows. Finally, exploration in domains with low-level action spaces is inherently difficult. Leveraging emergent temporal and action abstractions within autoregressive models could allow intrinsic rewards to operate over higher-level behaviors, improving the efficiency of information gathering \citep{kobayashi2025emergent}.

\vspace{-0.5em}
\paragraph{Comparing $\rsum$ and $\rdl$.}

Empirically, we found that $\rsum$ and $\rdl$ both performed well, but that the optimal one between the two depended on the environment in which they were deployed, suggesting that they can be complementary. Theoretical analysis in BED settings established that $\rsum$ bounds BIG from below, whereas $\rdl$ bounds it from above. 
The bounds for each are fundamentally different, $\rdl$ converging as $t \to \infty$, and $\rsum$ as $T\to\infty$. $\rdl$ technically recovers BIG but its bias vanishes together with the signal itself, and $\rsum$ should be more robust and approach BIG at finite $t$. However, our experiments show that this distinction may not manifest in simple environments, and at finite $T$ and $t$. Future work might be able to characterize the properties of environments that influence these factors. In addition, because the two bounds converge asymptotically from opposite sides, future work could explore hybridizing them to yield a tighter, more stable estimator of true learning progress.

\begin{ack}
The authors acknowledge Seijin Kobayashi for helpful feedback and discussions.
EE and SB are PhD students at the Université de Montréal \& Mila -- Quebec AI Institute.
GL is an Associate Professor at the Université de Montréal and a Core Academic Member at Mila -- Quebec AI Institute.
EE acknowledges support from Vanier Canada Graduate Scholarship \#492702.
GL acknowledges support from NSERC Discovery Grant RGPIN-2018-04821, the Canada Research Chair in Neural Computations and Interfacing, and a Canada-CIFAR AI Chair.
\end{ack}

\bibliographystyle{apalike}
\bibliography{references}

\clearpage
\begin{appendices}
\onecolumn
\crefalias{section}{appendix}   % For cleveref to name correctly
\crefalias{subsection}{appendix} % For cleveref to name correctly
\crefalias{subsubsection}{appendix} % For cleveref to name correctly
\counterwithin{figure}{section} % Reset counter for appendix figures
\counterwithin{table}{section}  % Reset counter for appendix tables

\section{Additional Prior Work on Intrinsic Curiosity}
\label{app:prior_work}

\paragraph{Surprisal in a Learned Latent Space.}

One effective approach for mitigating the problem of aleatoric noise is to measure surprisal in a learned latent space that only encodes \emph{predictable} information about the observation \citep{pathak2017curiosity,kim2019exploration,burda2018largescale}. However, these methods struggle in partially observed environments where the state representation depends on a long history, and also remain sensitive to sources of aleatoric noise that are produced by the policy's actions \citep{burda2018largescale}. Ultimately, the problem of how to train a latent space for surprisal-based exploration that is robust to aleatoric noise remains open. Since our investigation into the use of in-context learners for intrinsic curiosity leverages only $\rho$'s input/output interface for computing observation-level prediction errors, we do not consider methods that operate in latent space (although if in-context learners are trained in an appropriate latent space, the approach might be viable).

\paragraph{Novelty.}

An approach with a long history in reinforcement learning is to seek \emph{novel} states of the environment, as measured by visitation counts or density models \citep{hazan2019provably,strehl2008analysis,bellemare2016unifying,ostrovski2017count}. Novelty is deeply related to surprisal; while the two are not quite equivalent \citep{barto2013novelty}, they both suffer from the same fundamental problem: noisy regions of the environment with high observation entropy are a reliable source of both surprisal and novelty.

\paragraph{Empowerment.}

Empowerment is an alternative framework for intrinsic motivation that, instead of gathering observations that improve a model of the environment, aims to explore regions of the environment that are \emph{controllable} by the policy \citep{klyubin2005all,klyubin2005empowerment,salge2014empowerment}. Specifically, empowerment is defined as the channel capacity of the policy’s actuation channel: $I(S' ; A \mid S)$ where $I$ is mutual information, $S$ is a random variable for the history of all past states, $A$ is a random variable for the agent's current action, and $S'$ is a random variable for the state(s) the policy observes after some future horizon following the action ($S'$ can represent a single future state at a fixed horizon, or potentially a trajectory of such states). The distributions over $A$, $S$, and $S'$ depend on both the environment and the policy, and using empowerment as an intrinsic reward involves maximizing $I(S' ; A \mid S)$ with respect to the latter. Intuitively, this yields policies that learn to exert control over their environment by exploring regions of the state space where actions provide substantial information about future environment dynamics.

It is unclear how empowerment theoretically relates to the above methods that seek to improve a model of the environment. While exploring controllable regions of an environment might improve a model as a side effect (e.g., knowing what is and is not controllable involves understanding the environment's dynamics at some level), the two are not equivalent in general. While empowerment appears to be a viable exploration strategy empirically \citep{mohamed2015variational} and can be highly effective in modeling human exploration behaviour \citep{du2023what}, we instead focus here on intrinsic rewards that explicitly aim to improve a model of the environment.

\paragraph{Open-endedness.}

The aims of intrinsic curiosity closely align with those of research on ``open-endedness'' \citep{stanley2015greatness}, typically characterized as an ongoing process that continually generates new data, tasks, or behaviours that are interesting, learnable, novel, without being \emph{too} difficult relative to an agent's current capabilities. Most existing approaches operationalize open-endedness either through diversity/complexity heuristics \citep{wang2019paired,wang2024voyager} or through proxies intended to capture subjective human notions of interestingness \citep{ding2023quality,bradley2024qualitydiversity,zhang2024omni}.

A recent position paper attempts to mathematically formalize the concept of open-endedness \citep{hughes2024open}, defining it as a process that continually generates data that is (a) surprising and (b) learnable, where both are defined in terms of the prediction errors of a world model that has observed the data up to a particular timepoint. Notably, their learnability component closely resembles learning progress rewards. We further hypothesize that their explicit novelty term is redundant: if learning progress is sustained and does not stagnate, then novelty and surprisal are already implied, because continued progress requires exposure to information not yet captured by the model.

\section{Assumptions for Theoretical Results}
\label{app:assumptions}

The main text states six assumptions in compact form. We unpack each here, flagging the nuances that matter for the derivations.

\paragraph{(i) Markovian, stationary dynamics given $\theta$.}
For each fixed $\theta$, the transition kernel $p(s_{t+1} \mid s_t, a_t, \theta)$ is Markov in the state-action chain and time-independent. Stationarity rules out time-indexed families $\{P_t\}_t$ in which the conditional law itself depends on chain time. Markov dynamics given $\theta$ are what let us factorize $\rho$'s joint over future observations (\cref{lem:factor}); stationarity ensures that masking-based rewards ($\rsum$ and the NDIGO reward of \citet{azar2019world}) can be evaluated by anchoring their kernels at the actual chain times.

\paragraph{(ii) $\rho$ is exactly Bayesian over $\theta$.}
The sequence model emits the exact posterior-predictive $\rho(s_{t'} \mid h_{t'}) = \int p(s_{t'} \mid s_{t'-1}, a_{t'-1}, \theta)\, p(\theta \mid h_{t'})\, d\theta$, with prior $p(\theta)$ matching the distribution of training tasks. In practice $\rho$ is a finite-capacity neural network and exact Bayesian computation is an idealization; \cref{app:rho_approximation} discusses how approximation error propagates into reward biases.

\paragraph{(iii) Observation-only masking.}
The masked history $h_{t' \setminus t}$ replaces $s_t$ with a mask token while leaving every other state and every action intact. We assume $\rho$'s response to $h_{t' \setminus t}$ is the posterior-predictive obtained by Bayesian marginalization over the masked $s_t$, treating actions as still observed. This is a subtle assumption when $\rho$ is a learned sequence model: masked queries are out-of-distribution unless $\rho$ has been trained to handle them. See \cref{app:masking_training}.

\paragraph{(iv) Posterior consistency.}
The selection rule generates trajectories under which $p(\theta \mid h_t) \to \delta_{\theta^*}$ in probability as $t \to \infty$. This requires (a) $\theta$ identifiability from the achievable observation distributions under the policy and (b) sufficient coverage of the action space (each action taken often enough to inform $\theta$). Under standard regularity conditions (smooth parametric model, positive-definite per-arm Fisher information $\mathcal{I}(\theta) \succ 0$ uniformly on the parameter support), Bernstein--von-Mises gives the rate $\mathrm{Var}\!\bigl[p(\theta \mid h_t)\bigr] \sim \mathcal{I}(\theta^*)^{-1}/t$, controlling both the residual in \cref{thm:rsum} and the Jensen gap in \cref{thm:rdl-BED} at the $1/t$ scale. Strictly, the consistency assumption alone is what the framework requires; the Fisher-information regularity becomes relevant only when one wants asymptotic rates rather than asymptotic limits.

\paragraph{(v) Policy-generated actions.}
Actions are sampled from $\pi_\phi$ conditional on past observations. Because actions are not generated by the environment, conditioning on a future action does not update the posterior over $\theta$: $p(\theta \mid h_t, a_t, a_{t+1}, \ldots, a_{t'-1}) = p(\theta \mid h_t)$. This permits treating future actions as fixed inputs in the derivations, and is equivalent to the standard interventional treatment of actions in causal Markov chains.

\paragraph{(vi) Mixing given $\theta$.}
For each $\theta$ in the support of $p(\theta)$, $I(s_t; s_{t+k} \mid \theta, h_t, a_{t:t+k-1}) \to 0$ as $k \to \infty$. This rules out absorbing states and other non-ergodic dynamics under which the Markov chain never forgets $s_t$. Mixing is what makes the abductive bias of \cref{thm:rsum} (and its $\rgap$-generalized analogues) decay with the gap parameter $K$; without it, infinite-gap rewards still retain a bias from kernel-mediated coupling between $s_t$ and the deep future.

\subsection{The Prior \texorpdfstring{$p(\theta)$}{} and Meta-Learning}
\label{app:prior}

The prior $p(\theta)$ throughout this paper is the training distribution $p_{\mathrm{train}}(\theta)$ over MDPs --- $\rho$ is pretrained on trajectories from MDPs sampled from $p_{\mathrm{train}}$, and its posterior-predictive is assumed to be Bayesian \emph{under $p_{\mathrm{train}}$}. This is a meta-learning setup, and the framework's predictions extend to deployment insofar as the deployed environment lies in $\mathrm{supp}(p_{\mathrm{train}})$. We treat deployment mismatch on $\theta$ as a separate source of bias and adopt the train-prior framing throughout.

A more pointed meta-learning concern is that $\rho$'s training data --- pairs of trajectories and their corresponding posterior-predictives --- depends not only on $p_{\mathrm{train}}(\theta)$ but on the behavior of the policy that generates the training trajectories. Pretraining $\rho$ offline, separately from the deployment policy $\pi_\phi$, therefore requires the pretraining policy to cover the trajectory distribution adequately. This is mostly tractable in BED settings, where trajectories factorize across actions given $\theta$. Ensuring enough coverage of action selection during training should suffice. It becomes increasingly difficult as the domain expands, and especially in temporally rich BAMDPs where the trajectory distribution depends nonlinearly on the policy: a deployed reward-maximizing policy can easily probe regions that the pretraining policy under-sampled, leaving $\rho$ uncalibrated where it matters most. In principle, this can be alleviated with large enough training samples under enough random policies. Nevertheless, this is a practical matter to be considered.

This is a potential weakness of any offline-pretraining plus online-policy-training scheme: the predictive model is fixed during policy optimization, but policy iteration may drift outside $\rho$'s effective training support. On the other hand, the scheme opens the door to the use of massive pretrained but frozen foundation models (maybe even LLMs) to support active ICL in a variety of environments. Practical investigations of policy training in this regime, including the robustness of the method to departure from Bayesian exactness of $\rho$ (see next section), and  the question of action selection bias, is exciting future work.

Another avenue for the approach to have practical legs at scale would be that of joint training of $\pi$ and $\rho$. This could take the form of a meta-learning regime in which the predictor and the data-collection policy improve together so that $\rho$'s posterior-predictive remains accurate over the trajectories the policy actually generates. We visit this question in \cref{sec:discussion}.

\subsection{Inexactitude of \texorpdfstring{$\rho$}{}'s Bayesian Posterior Predictive}
\label{app:rho_approximation}

Assumption (ii) demands that $\rho$ be exactly Bayesian. In practice $\rho$ is a finite-capacity neural network trained by maximum likelihood on a finite dataset, so it differs from the exact Bayesian posterior-predictive $\rho^*$ in a controllable but non-zero way. This subsection discusses the consequences of inexactitude as a property of $\rho$ as a learner --- independent of any choice about how queries are conditioned on the trajectory. Issues specifically arising from masked queries during training or inference are deferred to \cref{app:masking_training}.

Writing $\Delta(Y) \coloneqq \KL(\rho^*(\cdot \mid Y) \,\|\, \rho(\cdot \mid Y))$ for the per-context KL between exact and approximate predictives at a conditioning $Y$, each reward inherits a bias that is a signed combination of such KLs. For example,
\[
  \mathbb{E}_{\rho^*}[\rsur_t \mid h_t]
    \;=\; I(s_t; \theta \mid h_t)
       \;+\; \Ent(s_t \mid h_t, \theta) \;+\; \Delta(h_t),
\]
showing that surprisal acquires an additional non-negative $\Delta(h_t)$ on top of the noisy-TV bias. More generally, rewards in $\calF$ formed from log-ratios of predictives at two conditionings $Y_+, Y_-$ pick up a differential $\bar\Delta(Y_+) - \bar\Delta(Y_-)$ rather than an absolute KL, since the two predictive errors partially cancel under log-subtraction.

Two structural features:
\begin{itemize}
  \item \textbf{$\rsur$'s bias is absolute.} A single non-negative KL $\Delta(h_t)$. Contexts where $\rho$ predicts poorly are systematically over-rewarded, on top of the aleatoric noisy-TV --- a model-error-seeking pathology specific to surprisal.
  \item \textbf{Log-ratio rewards have differential bias.} Rewards of the form $\log\rho(\cdot \mid Y_+) - \log\rho(\cdot \mid Y_-)$ inherit only the difference of $\rho$'s errors at $Y_+$ and $Y_-$. Uniform approximation error cancels; only the asymmetry between the two conditionings survives.
\end{itemize}

The differential structure of log-ratio rewards is a robustness property: an across-the-board imperfection in $\rho$ cancels in the log-ratio, leaving only the part of the error that distinguishes the two conditionings. Whether and how this differential is small in practice is a separate question --- it depends on how $\rho$ is trained and queried. \Cref{fig:gp-corr} in \cref{app:gp-worldmodel} shows the difference of log-predictive terms between an oracle observer and a trained PFN $\rho$ for the Gaussian process experiment. \Cref{fig:gp_bar} in the main text shows very little difference in reward training outcomes for the same experiment.

\subsection{Training \texorpdfstring{$\rho$}{} With vs. Without Masked Inputs}
\label{app:masking_training}

Computing $\rsum$ (and analogously the NDIGO reward of \citet{azar2019world}) requires evaluating $\rho$ on histories with one or more past states masked out. A sequence model trained with standard autoregressive next-token prediction will never have seen a mask token in its training data, so its behavior on masked queries is out-of-distribution. Two training strategies are natural:

\paragraph{Standard autoregressive training.} $\rho$ is trained on natural (unmasked) trajectories. At inference, predictives at masked conditioning sets reflect whatever extrapolation the model performs, with no guarantee of agreement with the Bayesian marginalization over the masked observation. This is the source of the differential bias described in \cref{app:rho_approximation}.

\paragraph{Masked training.} $\rho$ is trained with random masking of past observations during training, in the spirit of masked language modeling. $\rho$ then learns predictives that are well-calibrated with respect to the marginalized joint, making mask-baseline rewards potentially more reliable.

We remark that these considerations impact reward evaluation in the full BAMDP setting where parts of transition tuples $(s_{t-1},a_{t-1},s_t)$ need to be masked. However in the BED setting, masking simplifies to simply removing observations from pairs $(a_{t-1},s_t)$ and is not subject to the subtleties described above. Our experiments take place in the BED setting, but a more general implementation of ICL intrinsic rewards should consider the points above.

\section{Proofs for Theoretical Results}
\label{app:proofs}

\subsection{Interventional Factorization}
\label{app:factorization}

Several of our derivations rely on factorizing $\rho$'s joint over future observations into a product of one-step posterior-predictives. The factorization hinges on two structural features: (a) Markov dynamics given $\theta$ (assumption (i)), and (b) policy-generated actions that do not update the $\theta$-posterior (assumption (v)). The second is what makes actions \emph{interventional} in the causal sense: they enter the joint as fixed inputs, not as random variables to be modeled.

\begin{lemma}[Interventional factorization]
\label{lem:factor}
For $t_1 < t_2$ and any action sequence $a_{t_1:t_2-1}$,
\[
  \rho(s_{t_1:t_2} \mid h_{t_1}, a_{t_1:t_2-1})
  \;=\; \prod_{\tau=t_1}^{t_2}
    \rho(s_\tau \mid h_{t_1}, s_{t_1:\tau-1}, a_{t_1:\tau-1}),
\]
and the same identity holds with masked histories provided masked states lie outside $[t_1, t_2]$.
\end{lemma}

\begin{proof}
Under (i), the joint factors as $\prod_\tau p(s_\tau \mid s_{\tau-1}, a_{\tau-1}, \theta)$.  By assumption (v), conditioning on future actions does not update the posterior over $\theta$.  Marginalizing $\theta$ on both sides gives the product of posterior-predictives. Assumption (iii) extends to masked histories.
\end{proof}

This factorization is what allows $\rsum$'s telescoping representation $\rsum_t = \pmi(s_t; s_{t+1:T} \mid h_t, a_{t:T-1})$ to be used in the proof of \cref{thm:rsum}.

\subsection{Proof of \texorpdfstring{\cref{thm:impossibility}}{}}
\label{app:proof-impossibility}
\ThmImpossibility*
\paragraph{Remark.}
The argument is a marginal-indistinguishability counterexample: we construct two BAMDPs in $\mathcal{M}$ on which any $r_t \in \calF$ has the same unconditional expectation but for which the expected BIG differs. Since an unbiased per-history estimator of BIG forces matched unconditional expectations, the contradiction settles the theorem.

\paragraph{Proof.}

\paragraph{Step 1: Unconditional reduction.}
Suppose $r_t \in \calF$ satisfies the per-history identity $\mathbb{E}^M[r_t \mid h_t] = I^M(s_t; \theta \mid h_t)$ for every history $h_t$ and every $M \in \mathcal{M}$. Taking unconditional expectation,
\[
  \mathbb{E}^M[r_t] \;=\; \mathbb{E}^M\bigl[\,I^M(s_t; \theta \mid h_t)\,\bigr].
\]
We will exhibit $M_1, M_2 \in \mathcal{M}$ with $\mathbb{E}^{M_1}[r_t] = \mathbb{E}^{M_2}[r_t]$ but $\mathbb{E}^{M_1}\bigl[I(s_t; \theta \mid h_t)\bigr] \ne \mathbb{E}^{M_2}\bigl[I(s_t; \theta \mid h_t)\bigr]$, contradicting the per-history identity for at least one of them. We refer to such a pair $(M_1, M_2)$ --- or, since the kernel is fixed in our construction, the corresponding prior pair $(p_1, p_2)$ --- as a \emph{witness} to the theorem.

\paragraph{Step 2: $\mathbb{E}[r_t]$ depends only on the marginal trajectory distribution.}
Each $r_t = f\bigl((\rho(X_i \mid Y_i))_{i \in I}\bigr) \in \calF$ has finite \emph{touch horizon}
\[
  Q \;\coloneqq\; \max\bigl\{\tau : \tau \text{ is a chain index appearing in some } X_i \text{ or } Y_i\bigr\},
\]
i.e., the largest chain index referenced by the reward. Under Assumption~(ii), $\rho(X \mid Y) = P^M_{\mathrm{marg}}(X \mid Y)$ where $P^M_{\mathrm{marg}}$ is the marginal trajectory law of $M$ (with $\theta$ integrated out). Taking unconditional expectation,
\[
  \mathbb{E}^M[r_t] \;=\; \mathbb{E}_{P^M_{\mathrm{marg}}}\Bigl[ f\bigl((P^M_{\mathrm{marg}}(X_i \mid Y_i))_{i \in I}\bigr) \Bigr],
\]
which depends on $M$ only through $P^M_{\mathrm{marg}}(s_{1:Q}, a_{1:Q-1})$. In particular, two BAMDPs with identical marginal trajectory distributions on the first $Q$ steps yield identical $\mathbb{E}[r_t]$.

\paragraph{Step 3: A non-degenerate kernel family with prior-dependent BAMDPs.}
Take the BAMDP family with state space $\mathcal{S} = \{0, 1\}$, a single (dummy) action, initial state $s_1 \sim \mathrm{Bernoulli}(0.5)$, and transition kernel
\[
  p(s_{\tau+1} \ne s_\tau \mid s_\tau, \theta) \;=\; \theta,
  \qquad \theta \in (0, 0.5).
\]
The kernel is non-deterministic ($\theta > 0$) and not independent across time given $\theta$ ($\theta \ne 0.5$), so each fixed $\theta \in (0, 0.5)$ yields a kernel meeting the non-degeneracy condition of \cref{thm:impossibility}. Different priors $p(\theta)$ supported on $(0, 0.5)$ specify different BAMDPs in $\mathcal{M}$.

The marginal trajectory distribution of a sequence $s_{1:Q}$ depends only on the number of state flips $k = \#\{\tau : s_{\tau+1} \ne s_\tau\}$:
\[
  P_{\mathrm{marg}}^{p}(s_{1:Q})
  \;=\; \tfrac{1}{2} \int_0^{0.5} \theta^k (1 - \theta)^{Q-1-k}\, p(\theta)\, d\theta.
\]
The integrand is a polynomial in $\theta$ of degree at most $Q-1$. Expanding,
\[
  P_{\mathrm{marg}}^{p}(s_{1:Q})
  \;=\; \tfrac{1}{2} \sum_{j=0}^{Q-1} \beta_j(k)\, m_j(p),
\]
where $m_j(p) = \int \theta^j p(\theta)\, d\theta$ is the $j$-th moment of $p$ and the $\beta_j(k)$ are combinatorial constants. Hence $P_{\mathrm{marg}}^{p}(s_{1:Q})$ depends on $p$ only through its first $Q-1$ moments.

By the truncated Hausdorff moment problem \citep{schmudgen2017moment}, the set of probability measures on $(0, 0.5)$ matching any prescribed first $Q-1$ moments is convex with multiple distinct extreme points. Choose two distinct priors $p_1, p_2$ on $(0, 0.5)$ that match in their first $Q-1$ moments. Setting $M_1 = (\text{kernel}, p_1)$ and $M_2 = (\text{kernel}, p_2)$, both lie in $\mathcal{M}$, and by Step~2,
\[
  \mathbb{E}^{M_1}[r_t] \;=\; \mathbb{E}^{M_2}[r_t].
\]

\paragraph{Step 4: Expected BIG differs.}
Decompose
\[
  \mathbb{E}^M\bigl[I(s_t; \theta \mid h_t)\bigr]
  \;=\; \Ent^M_{\mathrm{marg}}(s_t \mid h_t)
       \;-\; \mathbb{E}_{p}\bigl[\Ent(s_t \mid s_{t-1}, \theta)\bigr].
\]
The marginal entropy $\Ent^M_{\mathrm{marg}}(s_t \mid h_t)$ is identical for $M_1$ and $M_2$ (Step~3, marginals match). The aleatoric term is
\[
  \mathbb{E}_{p}\bigl[\Ent(s_t \mid s_{t-1}, \theta)\bigr]
  \;=\; \mathbb{E}_{p}\bigl[\Ent_{\mathrm{bin}}(\theta)\bigr],
\]
where $\Ent_{\mathrm{bin}}(\theta) = -\theta\log\theta - (1-\theta)\log(1-\theta)$ is the binary entropy function. Since $\Ent_{\mathrm{bin}}$ is not a polynomial, its expectation under $p$ is not determined by any finite number of moments of $p$. By the truncated moment problem, we may choose $p_1, p_2$ matching in the first $Q-1$ moments but with $\mathbb{E}_{p_1}[\Ent_{\mathrm{bin}}(\theta)] \ne \mathbb{E}_{p_2}[\Ent_{\mathrm{bin}}(\theta)]$. Therefore $\mathbb{E}^{M_1}\bigl[I(s_t; \theta \mid h_t)\bigr] \ne \mathbb{E}^{M_2}\bigl[I(s_t; \theta \mid h_t)\bigr]$.

\paragraph{Conclusion.}
We have $\mathbb{E}^{M_1}[r_t] = \mathbb{E}^{M_2}[r_t]$ but $\mathbb{E}^{M_1}\bigl[I(s_t; \theta \mid h_t)\bigr] \ne \mathbb{E}^{M_2}\bigl[I(s_t; \theta \mid h_t)\bigr]$. The per-history identity $\mathbb{E}^M[r_t \mid h_t] = I^M(s_t; \theta \mid h_t)$ implies matched unconditional expectations, so it must fail on at least one of $M_1, M_2$. Since the construction works for any $r_t \in \calF$, the theorem follows. \qed

\subsection{Proof of \texorpdfstring{\cref{thm:rsur}}{}}
\label{app:proof-rsur}
\ThmRsur*
\begin{proof}
$\mathbb{E}[-\log\rho(s_t \mid h_t) \mid h_t] = \Ent(s_t \mid h_t)$.
By the conditioning-information identity,
$\Ent(s_t \mid h_t) = \Ent(s_t \mid h_t, \theta) + I(s_t; \theta \mid h_t)$.
Under (ii), $\rho$'s $\theta$-conditional predictive equals the true
kernel, so $\Ent(s_t \mid h_t, \theta) = \Ent(s_t \mid s_{t-1}, a_{t-1}, \theta)$.
\end{proof}

\subsection{Proof of \texorpdfstring{\cref{thm:rsum}}{}}
\label{app:proof-rsum}
\ThmRsum*
\begin{proof}
By \cref{lem:factor},
$\sum_{t'=t+1}^T \log\rho(s_{t'} \mid h_{t'}) =
\log\rho(s_{t+1:T} \mid h_t, s_t, a_{t:T-1})$ and
$\sum_{t'=t+1}^T \log\rho(s_{t'} \mid h_{t' \setminus t}) =
\log\rho(s_{t+1:T} \mid h_t, a_{t:T-1})$. Subtracting,
$\rsum_t = \pmi(s_t; s_{t+1:T} \mid h_t, a_{t:T-1})$. Take
expectation:
$\mathbb{E}[\rsum_t] = I(s_t; s_{t+1:T} \mid h_t, a_{t:T-1})$.

Applying the chain rule against $\theta$:
\begin{align*}
  I(s_t; s_{t+1:T} \mid h_t, a_{t:T-1})
  =\;& I(s_t; \theta \mid h_t, a_{t:T-1})\\
   -\;& I(s_t; \theta \mid h_t, a_{t:T-1}, s_{t+1:T})\\
   +\;& I(s_t; s_{t+1:T} \mid \theta, h_t, a_{t:T-1}) .
\end{align*}
Future actions do not update the $\theta$-posterior, so
$I(s_t; \theta \mid h_t, a_{t:T-1}) = I(s_t; \theta \mid h_t)$.
By Markov given $\theta$,
$s_t \perp s_{t+2:T} \mid s_{t+1}, \theta$, and so
$I(s_t; s_{t+1:T} \mid \theta, h_t, a_{t:T-1}) =
I(s_t; s_{t+1} \mid \theta, h_t, a_t)$
after dropping $d$-separated future actions.
\end{proof}

\begin{proof}[Vanishing residual term]
Under assumption (iv), $p(\theta \mid h_t, a_{t:T-1}, s_{t+1:T}) \to \delta_{\theta^*}$ in probability as $T \to \infty$; further conditioning on $s_t$ cannot move a delta, so the residual term vanishes.
\end{proof}

\subsection{The Log-Ratio Subclass and Its Chain-Rule Decomposition}
\label{app:log-ratio}

The proof of \cref{thm:impossibility} settles the impossibility for the entire class $\calF$ via the unconditional reduction. A complementary structural perspective is available for a specific sub-class of $\calF$ --- log-ratio rewards --- whose expected value reduces to an observation-space mutual information rather than to an entropy. This subsection isolates that sub-class and gives the chain-rule decomposition that underlies the named confound terms in $\rsum$'s decomposition (\cref{thm:rsum}) and in the $\rone(K)$ decomposition (\cref{app:rone}).

\paragraph{The log-ratio subclass $\calF_{\mathrm{LR}}$.}
For any linear combination of log likelihoods $r_t = \sum_i c_i \log\rho(X_i \mid Y_i) \in \calF$, taking expectation under $\rho^*$ gives
\[
  \mathbb{E}[r_t] \;=\; -\sum_i c_i \Ent(X_i \mid Y_i)
                \;=\; -\sum_i c_i \Ent(X_i)
                       \;+\; \sum_i c_i\, I(X_i; Y_i).
\]
The marginal-entropy contribution $\sum_i c_i \Ent(X_i)$ vanishes identically iff, for each distinct random variable $X$ appearing among the $X_i$'s, the coefficients on terms with that $X$ sum to zero. Rewards meeting this constraint can be written as a sum of \emph{matched log-ratios}
\[
  r_t \;=\; \sum_k \alpha_k \bigl[\log\rho(X_k \mid Y_k^+) - \log\rho(X_k \mid Y_k^-)\bigr].
\]
We denote this sub-class $\calF_{\mathrm{LR}}$. Membership:
\begin{itemize}
  \item $\rsum_t \in \calF_{\mathrm{LR}}$: each summand is a paired log-ratio with $X_k = s_{t'}$ for the same future state.
  \item $\rone_t(K) \in \calF_{\mathrm{LR}}$: a single paired log-ratio with $X = s_{t+K}$.
  \item $\rsur_t \notin \calF_{\mathrm{LR}}$: its single unpaired term $-\log\rho(s_t \mid h_t)$ has expectation equal to the conditional entropy $\Ent(s_t \mid h_t)$, with the marginal entropy of $s_t$ surviving.
\end{itemize}

\paragraph{Chain-rule decomposition.}
For $r_t \in \calF_{\mathrm{LR}}$, the expected reward reduces to a signed sum of observation-space conditional MIs:
\[
  \mathbb{E}[r_t] \;=\; \sum_j \alpha_j\, I(A_j; B_j \mid C_j),
\]
with $A_j, B_j, C_j$ subsets of the trajectory and $\theta$ absent from all arguments (since $\rho$ marginalizes over it). Each term decomposes via the chain rule against $\theta$:
\[
  I(A_j; B_j \mid C_j)
  \;=\; \underbrace{I(A_j; \theta \mid C_j)}_{\textnormal{signal-like}}
       \;-\; \underbrace{I(A_j; \theta \mid C_j, B_j)}_{\textnormal{residual}}
       \;+\; \underbrace{I(A_j; B_j \mid \theta, C_j)}_{\textnormal{abductive}}.
\]

\paragraph{Reading the three terms.}
\begin{itemize}
  \item Signal-like terms quantify how the observation subsets $A_j$ inform $\theta$ given the conditioning $C_j$.
  \item Residual terms measure the $\theta$-information that $A_j$ retains after $B_j$ is also observed.
  \item Abductive terms capture dependencies between observation subsets that operate through the kernel's dynamics rather than through their shared dependence on $\theta$.
\end{itemize}

This decomposition is exactly what produces the named offsets in \cref{thm:rsum} (one-step abductive and residual for $\rsum$) and in \cref{thm:rone} ($K$-step abductive and residual for $\rone(K)$). It does \emph{not} extend to rewards outside $\calF_{\mathrm{LR}}$: $\rsur$ in particular requires its own analysis via the entropy chain rule $\Ent(s_t \mid h_t) = \Ent(s_t \mid h_t, \theta) + I(s_t; \theta \mid h_t)$, yielding the BIG-plus-aleatoric form of \cref{thm:rsur}.

\paragraph{Relation to \cref{thm:impossibility}.}
For log-ratio rewards, the chain-rule decomposition gives a useful diagnostic of \emph{where} the bias relative to BIG comes from: typically a non-vanishing abductive or a non-vanishing residual. It is not, however, the engine of \cref{thm:impossibility}. The proof above (\cref{app:proof-impossibility}) operates at the level of unconditional expectations and works for the entire class $\calF$ uniformly, including non-log-ratio rewards, without invoking the chain-rule decomposition.

% \subsection{\texorpdfstring{$\rone(K)$}{} Decomposition}
\subsection{NDIGO reward Decomposition}
\label{app:rone}

The NDIGO reward of \citet{azar2019world} is the single-summand analogue of $\rsum$ at horizon $K$. We thus refer to it as $\rone(K)$, and define in the $\rho$ setting as:
\[
  \rone_t(K) \;\coloneqq\; \log\rho(s_{t+K} \mid h_t, s_t, a_{t:t+K-1})
                          - \log\rho(s_{t+K} \mid h_t, a_{t:t+K-1}).
\]

\begin{theorem}[Decomposition of $\rone(K)$]\label{thm:rone}
Under assumptions (i)--(iii) and (v),
\begin{align*}
  \mathbb{E}[\rone_t(K) \mid h_t, a_{t:t+K-1}]
  \;=\;\;& I(s_t; \theta \mid h_t)\\
       +\;& I(s_t; s_{t+K} \mid \theta, h_t, a_{t:t+K-1})\\
       -\;& I(s_t; \theta \mid h_t, a_{t:t+K-1}, s_{t+K}).
\end{align*}
\end{theorem}

\begin{proof}
By assumption (v), $\rho(s_t \mid h_t, a_{t:t+K-1}) = \rho(s_t \mid h_t)$, hence $\rone_t(K) = \pmi(s_t; s_{t+K} \mid h_t, a_{t:t+K-1})$ (where $\pmi$ is point-wise mutual information). Take expectation and apply the chain rule against $\theta$ as in the proof of \Cref{thm:rsum}.
\end{proof}

At finite $K$ the residual $I(s_t; \theta \mid h_t, a_{t:t+K-1}, s_{t+K})$ does not vanish: a single $K$-step-ahead observation does not identify $\theta$. $\rone(K)$ is therefore doubly biased relative to $I(s_t; \theta \mid h_t)$ at finite $K$ (abductive plus persistent residual), and under BED structure reduces to $I(s_t; \theta \mid h_t)$ minus a residual.  In the limit $K \to T - t$, $\rone(K)$ approaches a single-summand analogue of $\rsum$.

\subsection{Beyond Finite Touch Horizon: The Gap Reward}
\label{app:gap-reward}

The decomposition theorems above place $\rsum$ and $\rone(K)$ at \emph{opposite} structural extremes within the log-ratio subclass:
\begin{itemize}
  \item $\rsum_t$ queries the entire future block $s_{t+1:T}$ with \emph{no gap}. Its epistemic residual vanishes as $T \to \infty$ (the block contains infinite data, identifying $\theta$), but its one-step abductive $I(s_t; s_{t+1} \mid \theta, h_t, a_t)$ persists --- the block starts immediately at $t+1$, fully coupled to $s_t$ through the Markov kernel.
  \item $\rone_t(K)$ queries a \emph{single state} $s_{t+K}$ at gap $K$. For $K \to \infty$ in a mixing chain, the abductive $I(s_t; s_{t+K} \mid \theta, h_t, a_{t:t+K-1})$ vanishes; but the epistemic residual $I(s_t; \theta \mid h_t, a_{t:t+K-1}, s_{t+K})$ persists --- a single state cannot identify $\theta$.
\end{itemize}
$\rsum$ controls the residual; $\rone$ controls the abductive. Each pushes one of two independent structural dials. This invites a unification: query a block that simultaneously starts at a gap $K$ and extends to the trajectory horizon $T$.

\begin{definition}[Gap reward]\label{def:rgap}
For $1 \le K \le T - t$,
\[
  \rgap_t(K) \;\coloneqq\; \log\rho\bigl(s_{t+K:T} \mid h_t, s_t, a_{t:T-1}\bigr)
                          - \log\rho\bigl(s_{t+K:T} \mid h_t, a_{t:T-1}\bigr).
\]
\end{definition}

The two limits of \cref{def:rgap} recover the rewards above. At $K = 1$, the conditioning block extends from $t+1$ to $T$ and \cref{lem:factor} together with assumption (iii) yields $\rgap_t(1) = \rsum_t$ exactly. Restricting the block of $\rgap_t(K)$ to its first state $s_{t+K}$ recovers $\rone_t(K)$. $\rgap$ is the two-parameter family that interpolates between them by independently choosing a gap $K$ and a block size $T - t - K + 1$.

Beyond assumptions (i)--(v), the gap-reward analysis crucially uses assumption (vi): mixing of the chain given $\theta$, ruling out absorbing states and other non-ergodic dynamics under which the chain never forgets $s_t$.

\begin{theorem}[Decomposition of $\rgap$]\label{thm:rgap-decomp}
Under (i)--(iii) and (v),
\begin{align*}
  \mathbb{E}[\rgap_t(K) \mid h_t, a_{t:T-1}]
  \;=\;\;& I(s_t; \theta \mid h_t)\\
       +\;& \underbrace{I(s_t; s_{t+K:T} \mid \theta, h_t, a_{t:T-1})}_{\textnormal{abductive at gap } K}\\
       -\;& \underbrace{I(s_t; \theta \mid h_t, a_{t:T-1}, s_{t+K:T})}_{\textnormal{residual at block end } T}.
\end{align*}
Both correction terms are non-negative.
\end{theorem}

\begin{proof}
By \cref{lem:factor}, $\rgap_t(K) = \pmi(s_t; s_{t+K:T} \mid h_t, a_{t:T-1})$; by (v), $\rho(s_t \mid h_t, a_{t:T-1}) = \rho(s_t \mid h_t)$. Take expectation and apply the chain rule against $\theta$ as in the proof of \cref{thm:rsum}:
\begin{align*}
  I(s_t; s_{t+K:T} \mid h_t, a_{t:T-1})
   \;=\;&\;  I(s_t; \theta \mid h_t, a_{t:T-1})\\
   -\;&\;    I(s_t; \theta \mid h_t, a_{t:T-1}, s_{t+K:T})\\
   +\;&\;    I(s_t; s_{t+K:T} \mid \theta, h_t, a_{t:T-1}).
\end{align*}
By (v), $I(s_t; \theta \mid h_t, a_{t:T-1}) = I(s_t; \theta \mid h_t)$.
\end{proof}

\cref{thm:rgap-decomp} subsumes \cref{thm:rsum} ($K = 1$) and reveals that the abductive of $\rgap$ is a $K$-shifted analogue of the one-step abductive of $\rsum$. The decomposition immediately yields a positive counterpart to the impossibility theorem in the iterated infinite limit.

\begin{corollary}[$\rgap$ recovers BIG in the double limit]\label{cor:rgap-BIG}
Under (i)--(vi),
\[
  \lim_{K \to \infty}\;\lim_{T \to \infty}\;
    \mathbb{E}[\rgap_t(K) \mid h_t, a_{t:T-1}]
  \;=\; I(s_t; \theta \mid h_t).
\]
\end{corollary}

\begin{proof}
\emph{Inner limit ($T \to \infty$, $K$ fixed).} As $T \to \infty$ the conditioning block $s_{t+K:T}$ grows without bound. Under (iv), $p(\theta \mid h_t, a_{t:T-1}, s_{t+K:T}) \to \delta_{\theta^*}$ in probability, so the residual $I(s_t; \theta \mid h_t, a_{t:T-1}, s_{t+K:T}) \to 0$.

\emph{Outer limit ($K \to \infty$).} With the residual at zero, only the abductive remains. By (i), the dependence of $s_{t+K:T}$ on $s_t$ given $\theta$ funnels through $s_{t+K}$: $I(s_t; s_{t+K:T} \mid \theta, h_t, a_{t:T-1}) = I(s_t; s_{t+K} \mid \theta, h_t, a_{t:t+K-1})$ (data-processing along the post-$s_{t+K}$ chain). By (vi), this vanishes as $K \to \infty$.
\end{proof}

\paragraph{Remarks.}
\begin{itemize}
  \item \emph{Outside $\calF$.} The limit reward in \cref{cor:rgap-BIG} cannot be expressed as a function of finitely many likelihoods, hence not a member of $\calF$. This is consistent with \cref{thm:impossibility}: at every finite $T, K$, $\rgap_t(K) \in \calF_{\mathrm{LR}} \subset \calF$ and inherits the bias of \cref{thm:rgap-decomp}; only the asymptotic limit escapes $\calF$.
  \item \emph{Marginalization burden.} Computing $\rgap_t(K)$ requires $\rho$ to evaluate predictives on histories where $s_t$ is masked and the intermediate states $s_{t+1:t+K-1}$ are absent --- effectively asking $\rho$ to marginalize over a gap of $K-1$ unobserved states. Assumption (iii) --- that $\rho$'s response to such queries is the correctly marginalized posterior-predictive --- becomes increasingly strained as the gap grows; see \cref{app:masking_training}.
  \item \emph{Mixing failure.} Assumption (vi) fails for environments with absorbing states or other non-ergodic structure given $\theta$. In such cases the abductive does not decay with $K$ and the double limit need not yield BIG.
\end{itemize}

\subsection{Necessity boundlessly growing gaps for Log-Ratio Rewards}
\label{app:double-limit-necessity}

\cref{cor:rgap-BIG} shows that an infinite gap together with an infinite block is \emph{sufficient} to recover BIG within a log-ratio reward. We now prove the converse: the same double-limit geometry is also \emph{necessary}. Concretely, any countable linear combination of matched log-ratios that universally identifies BIG must place all of its weight at infinite gap, with each non-trivial block of infinite size.

Let $\calF_{\mathrm{LR},T}$ (cf. \cref{app:log-ratio}) be the set of estimators which are linear combinations of log ratios involving subsets of observations up to a horizon $T$, i.e.,
$$
\resizebox{0.95\textwidth}{!}{$\displaystyle
\calF_{\mathrm{LR},T}=\left\{\sum_{Z\subseteq\{t+1,\ldots,T\}} c_Z \Bigl[ \log\rho(s_{Z} \mid h_t, s_t, a_{t:T})
                                - \log\rho(s_{Z} \mid h_t, a_{t:T}) \Bigr]:c_Z\in\mathbb{R}\,,\forall Z\subseteq\{t+1,\ldots,T\}\right\}\,.$}
$$

The following theorem shows that for any sequence of estimators chosen from $\calF_{\mathrm{LR},T}$ which is asymptotically equal to the Bayesian Information Gain $I(s_t;\theta|h_t)$ as $T\to\infty$ in a universal fashion over all BAMDPs, this sequence must involve an unboundedly growing gap akin to the one discussed in \cref{cor:rgap-BIG}.

\begin{theorem}[Necessity of a growing gap]\label{thm:double-limit-necessity}
Let $\mathcal{M}$ be the class of BAMDPs satisfying (i)--(vi) whose transition kernels are non-deterministic and not independent across time given $\theta$. Suppose that $(R_{t,T})_T$ is a sequence of estimators such that
\begin{itemize}
    \item $R_{t,T}\in\calF_{\mathrm{LR},T}$ for all $T$, i.e.,
    $$R_{t,T}=\sum_{Z\subseteq\{t+1,\ldots,T\}} c_{Z,T} \Bigl[ \log\rho(s_{Z} \mid h_t, s_t, a_{t:T})
                                - \log\rho(s_{Z} \mid h_t, a_{t:T}) \Bigr]\,,$$
    \item $\lim_{T\to\infty}\mathbb{E}^M[R_{t,T} \mid h_t, a_{t:T}] = I(s_t; \theta \mid h_t)$ for every $M \in \mathcal{M}$ and every $h_t, a_{t:\infty}$, and
    \item there exist constants $C > 0$ and $B \ge 1$ such that $|W_{k,T}| \le \sum_{k_Z=k} |c_{Z,T}| \le C B^k$ for all $k$ and $T$.\footnote{It is worth noting that this is a mild assumption. In practice, numerically stable estimators must satisfy this.}
\end{itemize}
If we define for each block of observations $Z\subset\{t+1,\ldots\}$ the gap $k_Z \coloneqq \min\{k \ge 1 : t+k \in Z\}$ and write $$W_{k,T} \coloneqq \sum_{\substack{Z \subseteq\{t+1\,\ldots,T\}:\\ k_Z = k}} c_{Z,T}\,,$$ then
$$\lim_{T\to\infty}W_{k,T}=0\,.$$
In other words, as $T\to\infty$, almost all the weight $\sum_Z c_{Z,T}$ ``will become concentrated on blocks of infinite gap''.
\end{theorem}

\begin{proof}
\textbf{Step 1: Chain-Rule Decomposition.}\\
By definition, each matched log-ratio term in $\calF_{\mathrm{LR},T}$ evaluates to a pointwise mutual information. Taking the expectation of $R_{t,T}$ and applying the chain rule against the environment parameter $\theta$, we obtain:
\begin{align*}
\mathbb{E}^M[R_{t,T} \mid h_t, a_{t:T}] &= \sum_{Z} c_{Z,T} I(s_t; s_{Z} \mid h_t, a_{t:T}) \\
&= \Big(\sum_{Z} c_{Z,T}\Big) I(s_t; \theta \mid h_t) + \sum_{Z} c_{Z,T} A_Z - \sum_{Z} c_{Z,T} E_Z\,,
\end{align*}
where $A_Z \coloneqq I(s_t; s_Z \mid \theta, h_t, a_{t:T})$ is the abductive term and $E_Z \coloneqq I(s_t; \theta \mid h_t, a_{t:T}, s_Z)$ is the epistemic residual.

\textbf{Step 2: Isolating the Abductive Bias.}\\
The abductive term $A_Z$ captures dependencies generated purely by the known transition kernel, relying on the prior $p(\theta)$ only through the posterior $p(\theta \mid h_t)$. We may evaluate the limit on an environment $M \in \mathcal{M}$ whose prior is a degenerate point mass at some $\theta^*$. For such a prior, there is no uncertainty about the environment, so the true Bayesian Information Gain $I(s_t; \theta \mid h_t) = 0$ and all epistemic residuals $E_Z = 0$ identically.
The universal convergence hypothesis then reduces to:
\begin{equation}\label{eq:abductive_limit}
\lim_{T\to\infty} \sum_{Z} c_{Z,T} A_Z(\theta^*) = 0 \quad \text{for every } \theta^*\,.
\end{equation}

\textbf{Step 3: The Markov Funnel.}\\
Under assumption (i), conditional on the exact parameter $\theta^*$, the environment is Markovian. By the data-processing inequality, the dependence of any block $s_Z$ on the state $s_t$ must funnel strictly through the chronologically earliest state in that block, $s_{t+k_Z}$. Dropping the separated future actions, we have:
$$ A_Z(\theta^*) = I(s_t; s_{t+k_Z} \mid \theta^*, h_t, a_{t:t+k_Z-1}) \coloneqq I_{k_Z}(\theta^*)\,. $$
Grouping the terms in \eqref{eq:abductive_limit} by their gap $k_Z=k$, we rewrite the limit in terms of the gap weights $W_{k,T}$:
\begin{equation}\label{eq:gap_limit}
\lim_{T\to\infty} \sum_{k=1}^{T} W_{k,T} I_k(\theta^*) = 0 \quad \text{for every } \theta^*\,.
\end{equation}

\textbf{Step 4: AR(1) Instantiation and Power Series Extraction.}\\
To evaluate \eqref{eq:gap_limit}, we instantiate $M$ with the Gaussian AR(1) family parameterized by $\alpha \in (0,1)$, for which the mutual information across a gap $k$ is $I_k(\alpha) = -\frac{1}{2} \log(1-\alpha^{2k})$. Letting $x = \alpha^2 \in (0,1)$, we define the sequence of functions:
$$ F_T(x) \coloneqq \sum_{k=1}^{T} W_{k,T} \big(-\log(1-x^k)\big)\,. $$
By hypothesis, $\lim_{T\to\infty} F_T(x) = 0$ for all $x \in (0,1)$.
We extend $F_T(x)$ to the complex plane as $F_T(z)$, which is analytic on the open unit disk $|z|<1$. Expanding the principal branch of the logarithm in its Maclaurin series $-\log(1-z^k) = \sum_{m=1}^\infty \frac{z^{km}}{m}$, we express $F_T(z)$ as:
$$ F_T(z) = \sum_{k=1}^{T} W_{k,T} \sum_{m=1}^\infty \frac{z^{km}}{m} = \sum_{n=1}^\infty \frac{V_{n,T}}{n} z^n\,, $$
where we have grouped terms by $n=km$ and defined $V_{n,T} \coloneqq \sum_{d|n,\; d<T} d W_{d,T}$.

Now recall that the theorem statement assumes that there exist constants $C > 0$ and $B \ge 1$ such that $|W_{k,T}| \le \sum_{k_Z=k} |c_{Z,T}| \le C B^k$ for all $k$ and $T$. 
For $z$ strictly inside the unit disk, we may bound the logarithmic term using its Maclaurin series: $|-\log(1-z^k)| \le \sum_{m=1}^\infty |z|^{km}/m \le |z|^k / (1-|z|)$. Applying our exponential bound to the weights yields:
$$ |F_T(z)| \le \sum_{k=1}^T |W_{k,T}|\, |-\log(1-z^k)| \le \frac{C}{1-|z|} \sum_{k=1}^\infty (B|z|)^k\,. $$
This geometric series converges absolutely for $B|z| < 1$. Thus, defining a bounding radius $R \coloneqq \min(1/2, 1/B)$, the sequence of functions $\{F_T(z)\}$ is uniformly bounded by a finite constant on any closed ball $B(0, r)$ with $r < R$. 

Because $F_T(x) \to 0$ pointwise on the real interval $(0, R)$, and this interval contains an accumulation point strictly within $B(0, R)$, the Vitali-Porter Convergence Theorem \citep[Section 2.4]{schiff1993normal} guarantees that $F_T(z)$ converges uniformly to $0$ on any strictly smaller closed contour $\gamma$ enclosing the origin.
By Cauchy's Integral Formula, the $n$-th Maclaurin coefficient is given by:
$$ \frac{V_{n,T}}{n} = \frac{1}{2\pi i} \oint_\gamma \frac{F_T(z)}{z^{n+1}} \mathrm{d}z\,. $$
Since $F_T(z)$ converges uniformly to $0$ on $\gamma$, we may pass the limit inside the integral, yielding:
$$ \lim_{T\to\infty} \frac{V_{n,T}}{n} = 0 \implies \lim_{T\to\infty} V_{n,T} = 0 \quad \text{for all } n \ge 1\,. $$

\textbf{Step 5: Strong Induction.}\\
We prove $\lim_{T\to\infty} W_{k,T} = 0$ by strong induction on $k$.
For the base case $k=1$, we have $V_{1,T} = W_{1,T}$ (since $d<T$ trivially for $T>1$), and thus $\lim_{T\to\infty} W_{1,T} = 0$.
For the inductive step, assume that $\lim_{T\to\infty} W_{d,T} = 0$ for all proper divisors $d < n$. By definition, for $T>n$:
$$ V_{n,T} = n W_{n,T} + \sum_{\substack{d|n \\ d < n}} d W_{d,T}\,. $$
Taking the limit as $T \to \infty$:
$$ 0 = n \lim_{T\to\infty} W_{n,T} + \sum_{\substack{d|n \\ d < n}} d \underbrace{\Big(\lim_{T\to\infty} W_{d,T}\Big)}_{=0} \implies \lim_{T\to\infty} W_{n,T} = 0\,. $$
This holds for all integers $n \ge 1$, proving that all finite gap weights strictly vanish as $T \to \infty$. Consequently, any estimator universally approaching BIG must concentrate its weight strictly at an unbounded, infinite gap.

\end{proof}

\paragraph{Remark (relation to \cref{thm:impossibility}).}
The earlier impossibility theorem covers $\calF$ at every finite touch horizon $Q$: at finite $T$, no member of $\calF$ identifies BIG, even without log-ratio structure. \cref{thm:double-limit-necessity} extends the negative result \emph{into} the limit: even allowing countable log-ratio combinations and asymptotic $T$, the only escape is a structural double limit in which both the gap and the block grow without bound. As mentioned in the main text, this imposes severe implementation limitations since predictive models $\rho$ are well suited for long sequence modeling (large $T$) but less so for marginalization over unobserved sequence blocks (large gaps).

\subsection{Proof of \texorpdfstring{\cref{cor:rsum-BED}}{}}
\label{app:proof-rsum-BED}
\CorRsumBED*
\begin{proof}
Under BED structure, given $\theta$ and the action sequence, $s_t \sim p(\cdot \mid a_{t-1}, \theta)$ and $s_{t+1} \sim p(\cdot \mid a_t, \theta)$ are independent. Hence $I(s_t; s_{t+1} \mid \theta, h_t, a_t) = 0$, and the abductive term in \cref{thm:rsum} drops out.
\end{proof}

\subsection{Proof of \texorpdfstring{\cref{thm:rdl-BED}}{}}
\label{app:proof-rdl-BED}

\ThmRdlBED*
\begin{proof}
Under BED, $p(s_t \mid s_{t-1}, a_{t-1}, \theta) = p(s_t \mid a_{t-1}, \theta)$, so the integral in \cref{eq:rdl} simplifies:
\[
  \int p(s_t \mid a_{t-1}, \theta)\, p(\theta \mid h_t, s_t)\, d\theta
  \;=\; \mathbb{E}_{p(\theta \mid h_t, s_t)}[p(s_t \mid a_{t-1}, \theta)]
  \;=\; \rho(s_t \mid h_t)\, \mathbb{E}_{p(\theta \mid h_t, s_t)}[L(\theta)],
\]
where the last equality uses the definition $L(\theta) = p(s_t \mid a_{t-1}, \theta)/\rho(s_t \mid h_t)$. Hence
\[
  \rdl_t \;=\; \log\mathbb{E}_{p(\theta \mid h_t, s_t)}[L(\theta)].
\]
By Jensen's inequality,
\[
  \rdl_t \;=\; \mathbb{E}_{p(\theta \mid h_t, s_t)}[\log L(\theta)] + \mathcal{J}_t,
\]
with $\mathcal{J}_t \ge 0$ the Jensen gap. Bayes' rule gives $L(\theta) = p(\theta \mid h_t, s_t)/p(\theta \mid h_t)$, so $\mathbb{E}_{p(\theta \mid h_t, s_t)}[\log L(\theta)] = \KL\!\bigl(p(\theta \mid h_t, s_t) \,\|\, p(\theta \mid h_t)\bigr)$. Taking expectation over $s_t$, the KL term integrates to $I(s_t; \theta \mid h_t)$ (Bayesian-surprise identity).
\end{proof}

\paragraph{Remark.} \Cref{thm:rdl-BED} shows an important limitation of $\rdl$. While technically, $\rdl$ does recover BIG asymptotically, it does so as $t \to \infty$ which also implies that BIG $I(s_t;\theta \mid h_t)$ itself vanishes. This means that $\rdl$ biases vanish at the same time as the BIG signal itself does. In contrast, $\rsum$ also recovers BIG asymptotically, but it does so as $T \to \infty$, meaning that for long trajectories, one can hope to recover minimally biased BIG signal for finite $t$. However, the practical effect of finite-time biases may or may not be consequential for a useful policy, depending on the environment. Experiments in the main text illustrate this point.

\section{Implementing \texorpdfstring{$\rdl$}{} in General BAMDPs vs. BED}
\label{app:rdl-in-bamdbs-vs-beds}

The reward $\rdl$ defined in \cref{eq:rdl} contains an integral term
\[
  \int p(s_t \mid s_{t-1}, a_{t-1}, \theta)\, p(\theta \mid h_t, s_t)\, d\theta
\]
that pairs the kernel at the trajectory input $(s_{t-1}, a_{t-1})$ with a posterior conditioned on the subsequent observation $s_t$. We explain here why this combination is generically not implementable through $\rho$'s standard predictive interface, and how the BED restriction restores implementability via a counterfactual action commitment.

\paragraph{Why no $\rho$-implementation in general BAMDPs.}
A standard $\rho$ query $\rho(s' \mid h)$ evaluates the kernel at the most recent $(s, a)$ in the conditioning history $h$, averaged under the posterior $p(\theta \mid h)$ that this same history induces. The integrand above departs from this template in a structural way: the kernel input is $(s_{t-1}, a_{t-1})$ from time $t-1$, while the posterior is $p(\theta \mid h_t, s_t)$ from after $s_t$ has been observed. In a general MDP, where the kernel genuinely depends on the chain-position state, no rearrangement of conditioning subsets in $\rho$'s query interface produces this combination --- evaluating $\rho$ at $(s_{t-1}, a_{t-1})$ gives the posterior only up to time $t-1$, and evaluating $\rho$ after $s_t$ shifts the kernel input as well. Implementing $\rdl$ in general BAMDPs therefore requires explicit Bayesian-update machinery beyond $\rho$'s forward-pass interface.

\paragraph{BED counterfactual identity.}
In BED the transition kernel has no $s_{t-1}$ dependence, $p(s_t \mid s_{t-1}, a_{t-1}, \theta) = p(s_t \mid a_{t-1}, \theta)$, so the integrand becomes
\[
  \int p(s_t \mid a_{t-1}, \theta)\, p(\theta \mid h_t, s_t)\, d\theta
  \;=\; \rho\bigl(s_{t+1} = s_t \,\big|\, h_t, s_t,\, a_t = a_{t-1}\bigr),
\]
recovering \cref{eq:rdl-counterfactual}. The right-hand side is a standard $\rho$ query: condition on the history $(h_t, s_t)$, commit to the action $a_t = a_{t-1}$ (a counterfactual on the action choice, valid because actions are policy-generated and thus interventional), and read off the predictive probability of the next outcome equaling the observed $s_t$. This makes $\rdl$ implementable in BED at the cost of one hypothetical rollout step in $\rho$'s context.

The key takeaway: BED collapses the chain-time mismatch by removing the $s_{t-1}$ dependence in the kernel --- an obstruction that no rearrangement of $\rho$'s queries can resolve in a general MDP. It could be possible to address this issue with a modified sequence model architecture that allows multiple channels for the same sequence location, together with a gating mechanism to implement a form of counterfactual ``superposition''. This remains a non-trivial design.

\section{Experimental Details}
\label{app:experimental_details}

\subsection{Gaussian Processes} 

\subsubsection{Environment}\label{app:gp}

The Gaussian Process (GP) environment presents the agent with an unknown function $f:\mathbb{R}^2 \to \mathbb{R}$ that must be explored through sequential point queries.
At the start of each episode, a function $f$ is sampled from a Gaussian process prior $f \sim \mathcal{GP}(0, k)$, where $k$ is the \emph{rational quadratic} kernel:
\begin{equation}
k(\mathbf{x}, \mathbf{x}') = \sigma^2 \left(1 + \frac{\|\mathbf{x} - \mathbf{x}'\|^2}{2 \alpha \ell^2}\right)^{-\alpha},
\end{equation}
with $\alpha = 1$, lengthscale $\ell = 4.0$, and $\sigma^2 = 1$.

Since a GP sample is an infinite-dimensional function, we represent it in practice via \emph{finite-grid sampling}: we place $R=10$ equally spaced points along each of the $d$ input dimensions in $[-x_{\max},, x_{\max}]$ and take their Cartesian product, yielding a grid $\mathbf{X}_{grid}$ of $N = R^d$ locations. To draw a sample $\mathbf{y}_{\mathrm{grid}}$ from this multivariate Gaussian we use the standard Cholesky method. Function values at arbitrary query points 
are then obtained via GP posterior interpolation. We use $x_{\max} = 10$.

When the agent queries location $\mathbf{x}_t$, it observes:
\begin{equation}
y_t = f(\mathbf{x}_t) + \epsilon_t, \qquad \epsilon_t \sim \mathcal{N}\bigl(0, \sigma_{\text{noise}}^2(\mathbf{x}_t)\bigr),
\end{equation}
where the noise variance $\sigma_{\text{noise}}^2(\mathbf{x})$ is \emph{spatially varying} according to a fixed checkerboard noise map. $\mathbf{X}_{grid}$ is partitioned into an $8 \times 8$ grid of cells; Tiles alternate between high-noise ($\sigma_{\mathrm{noise}}^2 = 5$) and low-noise ($\sigma_{\mathrm{noise}}^2 \approx 0$) regions. This creates a heteroscedastic noise landscape that an effective exploration policy must learn to navigate.

% \paragraph{Action Space}
% The input domain $[-10, 10]^2$ is discretized into an $8 \times 8$ uniform grid of $64$ query locations. At each step $t$, the agent selects a grid index $a_t \in \{0, \ldots, 63\}$ corresponding to a fixed location $\mathbf{x}_{a_t}$ on this grid.

\subsubsection{Bayesian Oracle World Model} \label{app:gp-bayes}

The Bayesian oracle leverages the closed-form GP posterior. Given a history of observations $\{(\mathbf{x}_i, y_i)\}_{i=0}^{t-1}$, the posterior predictive distribution for a new query $\mathbf{x}_t$ is Gaussian:
\begin{equation}
p(y_t \mid \mathbf{x}_t, h_t) = \mathcal{N}\bigl(y_t;\, \mu_t, \sigma_t^2\bigr),
\end{equation}
where the posterior mean and variance are given by the standard GP regression formulae:
\begin{align}
\mu_t &= \mathbf{k}_*^\top (K + \Sigma_{\text{noise}})^{-1} \mathbf{y}_{<t}, \\
\sigma_t^2 &= k(\mathbf{x}_t, \mathbf{x}_t) - \mathbf{k}_*^\top (K + \Sigma_{\text{noise}})^{-1} \mathbf{k}_*,
\end{align}
with $\mathbf{k}_* = [k(\mathbf{x}_i, \mathbf{x}_t)]_{i=0}^{t-1}$, $K_{ij} = k(\mathbf{x}_i, \mathbf{x}_j)$, and $\Sigma_{\text{noise}} = \operatorname{diag}(\sigma_{\text{noise}}^2(\mathbf{x}_0), \ldots, \sigma_{\text{noise}}^2(\mathbf{x}_{t-1}))$. 
% The Cholesky decomposition $L = \operatorname{chol}(K + \Sigma_{\text{noise}})$ is used for numerically stable inversion. 

For the \emph{predictive} distribution used by the surprisal-based rewards, the predictive variance additionally includes the noise at the query point:
\begin{equation}
\tilde{\sigma}_t^2 = \sigma_t^2 + \sigma_{\text{noise}}^2(\mathbf{x}_t).
\end{equation}

\paragraph{Per-step Surprisal}
The observation surprisal is the negative log-likelihood under the predictive Gaussian:
\begin{equation}
-\log p(y_t \mid \mathbf{x}_t, h_t) = \frac{1}{2}\log(2\pi) + \frac{1}{2}\log \tilde{\sigma}_t^2 + \frac{(y_t - \mu_t)^2}{2\tilde{\sigma}_t^2}.
\end{equation}

\paragraph{Per-step Information Gain (Bayesian Surprise)}
The information gain reward uses the KL divergence between the GP posterior and the predictive observation distribution. In exact form:
\begin{equation}
\text{G}_t = \frac{1}{2}\log\!\left(1 + \frac{\sigma_t^2}{\sigma_{\text{noise}}^2(\mathbf{x}_t)}\right) - \frac{\sigma_t^2}{2(\sigma_t^2 + \sigma_{\text{noise}}^2(\mathbf{x}_t))} + \frac{\sigma_t^2 \cdot (y_t - \mu_t)^2}{2(\sigma_t^2 + \sigma_{\text{noise}}^2(\mathbf{x}_t))^2},
\end{equation}
where $\sigma_t^2$ is the latent GP posterior variance (without observation noise).

\subsubsection{In-Context PFN World Model} \label{app:gp-worldmodel}

\paragraph{Data-Generating Process}
The PFN is trained to predict the next observation $y_t$ in-context from the history of observed pairs. For each training example, a function $f$ is sampled from the same GP prior as the environment (rational quadratic kernel with $\alpha=1$, $\ell=4.0$, spatially-varying noise), and $T_{\text{train}} = 100$ random query locations are drawn uniformly from $[-10, 10]^2$.

\paragraph{Model Architecture}
The PFN uses a causal Transformer decoder. Each past observation $(x_j, y_j)$ is embedded by summing a linear projection of $x_j$ and a linear projection of $y_j$ into a single ``pair'' token. This pair token is interleaved with an isolated embedding of the next
query $x_{j+1}$, yielding a sequence of length $2(t{-}1)$.
The causal Transformer decoder has model dimension $d = 512$, MLP hidden dimension $F = 1024$, $N = 12$ layers, and $H = 8$ attention heads with head dimension $d/H = 64$. The model uses LayerNorm and GELU activations.  
The Transformer output at each query-token position (odd positions) is projected via a learned linear layer to produce a $2$-dimensional output $(\hat{\mu}_t, \hat{\sigma}_t)$, representing a predicted mean and log-standard-deviation.

\paragraph{Training}
The  training objective is the mean negative log-likelihood of the observed $y$ values under the predicted Gaussian, with a variance-weighted loss to encourage better calibration \citep{seitzer2022on}:
\begin{equation}
\mathcal{L} = \frac{1}{(T-1)B}\sum_{b=1}^{B}\sum_{t=1}^{T-1} \bigl[\operatorname{sg}(\hat{\sigma}_t^{2})\bigr]^{\beta} \cdot \left[\frac{1}{2}\log(2\pi\hat{\sigma}_t^2) + \frac{(y_t - \hat{\mu}_t)^2}{2\hat{\sigma}_t^2}\right],
\end{equation}
where $\operatorname{sg}(\cdot)$ denotes stop-gradient, $\hat{\sigma}_t^2 = \exp(2\hat{\sigma}_t)$ is the predicted variance, and $\beta = 1$. The model is optimized using Adam with a learning rate of $10^{-4}$, batch size $B = 128$, for $10^6$ training steps.

We collect random trajectories from Gaussian processes, and plot the correlation between the predictions of the Bayesian oracle world model and the trained PFN world model in \cref{fig:gp-corr}. The trained PFN is effective at approximating the exact likelihood and, as evidenced by other results, is effective even in settings where the oracle cannot be computed.   

\begin{figure}[ht]
    \centering
    \includegraphics[width=0.8\textwidth]{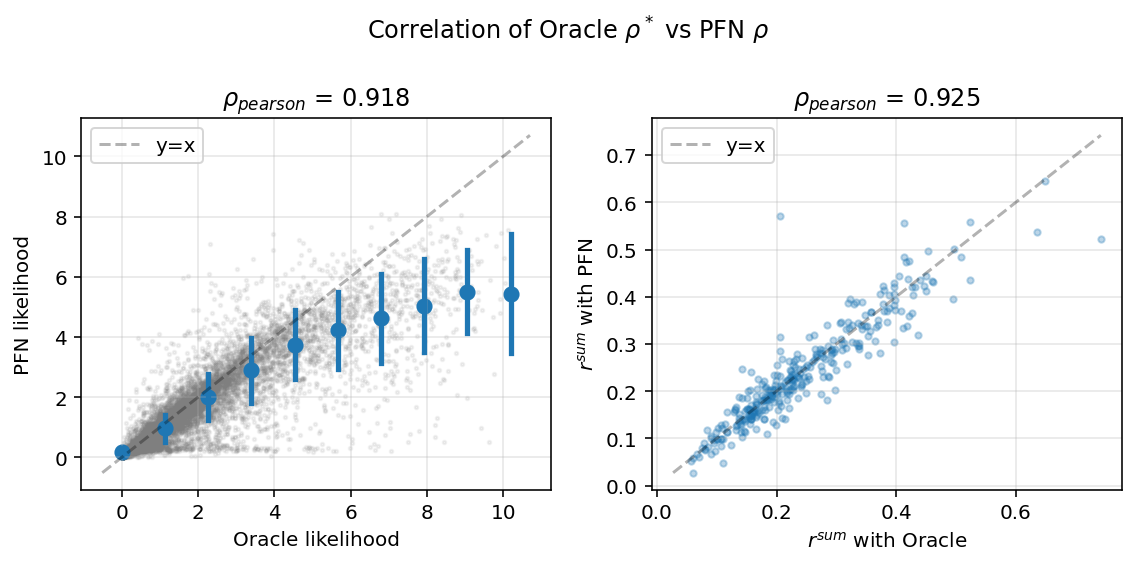}
    \caption{The correlation between the values computed by an exact bayesian oracle vs our learned PFN for likelihood and reward of randomly generated trajectories from Gaussian processes.}
    \label{fig:gp-corr}
\end{figure}

\subsubsection{Policy}

\begin{figure}[ht]
    \centering
    \includegraphics[width=0.4\textwidth]{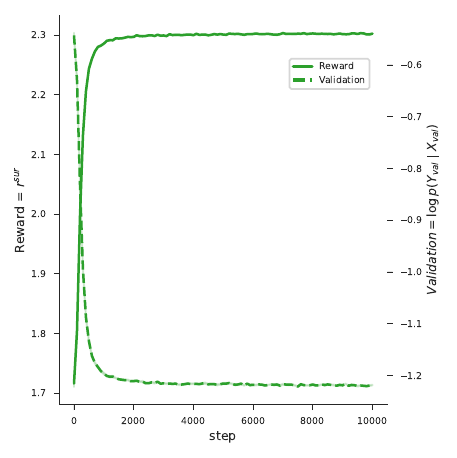}
    \caption{Training reward vs Validation log-probability over the course of training a policy for Gaussian Process. This highlights that in noisy environments, $\rsur$ can behave oppositely to the intended effect of learning the environment's dynamics.}
    \label{fig:gp-rsur-train}
\end{figure}

The policy is trained using REINFORCE with a learned value baseline. At each training step, $B = 32$ trajectories of length $T = 26$ are collected by rolling out the current policy in the environment. A single trajectory-level reward $R$ is computed from the world model, and the policy is updated with a single gradient step per batch. The REINFORCE objective combines a policy gradient loss and a value baseline loss:
\begin{equation}
\mathcal{L} = \mathcal{L}_\pi + \lambda_V \mathcal{L}_V,
\end{equation}
where $\lambda_V = 0.5$.

\paragraph{Model Architecture}
The policy is a sequence model that processes the history of observations recurrently. The backbone is a Griffin model \citep{de2024griffin}, a recurrent architecture based on linear recurrences (LRU). The Griffin configuration uses model dimension $d = 128$, MLP dimension $256$, LRU dimension $128$, $4$ attention heads, and $2$ recurrent layers.

At each step $t$, the observation $o_t = (\mathbf{x}_t, y_t)$ is encoded by concatenating $\mathbf{x}_t \in \mathbb{R}^2$ and $y_t \in \mathbb{R}$, yielding a 3-dimensional vector, which is projected to the model dimension $d = 128$ via a learned linear layer.

The sequence model output at each step is mapped to logits over the $64$ discrete grid locations via a single linear layer. 
A separate linear layer maps the sequence model output to a scalar value estimate $V(s_t) \in \mathbb{R}$.

\subsection{Decomposition of rewards}
\label{app:decomposition}

In \cref{sec:BED}, we analyzed the theoretical behaviour of curiosity rewards in BED environments. \cref{cor:rsum-BED} showed that $\rsum$ decomposes into $\text{BIG} - \text{Residual}$ terms, while \cref{thm:rdl-BED} showed that $\rdl$ decomposes into $\text{BIG} + \text{JensenGap}$ terms. In our Gaussian process environment using an Bayesian oracle observer, all terms can be computed analytically. After having trained a policy with either $\rsum$ or $\rdl$, we can therefore check whether the reward empirically decomposes into the expected terms that our theory predicts.

\begin{figure}[ht]
    \centering
    \includegraphics[width=\textwidth]{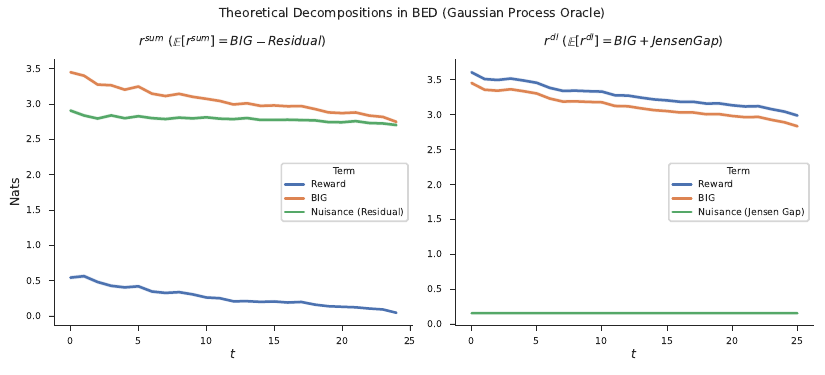}
    \caption{Empirical decompositions of (i) $\rsum$ into BIG and a subtracted residual term and (ii) $\rdl$ into an additive Jensen gap, in a Gaussian process environment with oracle in-context learner.}
    \label{fig:decomposition}
\end{figure}

\cref{fig:decomposition} plots the resulting reward, BIG, and nuisance term as a function of time, averaged across 10 training seeds and 1000 unrolled trajectories. We find that both rewards decompose \emph{exactly} into their theoretical constituents, thus validating our theoretical analyses. Indeed, plotting $\text{BIG} - \text{Residual}$ overlaps exactly with $\rsum$, and plotting $\text{BIG} + \text{JensenGap}$ overlaps exactly with $\rdl$.

\subsection{Mastermind}

At the start of each episode, the environment samples a secret code $\mathbf{c}^* \in \{0, 1, \ldots, C-1\}^L$ uniformly at random, where $C$ is the number of colours and $L$ is the code length. In our experiments we set $C = 6$ and $L = 7$, yielding $6^7 = 279{,}936$ possible codes. We hold out 10\% of the possible codes for evaluation. 
  
  At each timestep $t$, the agent submits a guess $\mathbf{g}_t \in \{0, 1, \ldots, C-1\}^L$ and receives a response $(b_t, w_t)$ where:
\begin{itemize}
  \item Black peg $b_t \in \{0, \ldots, L\}$ is the number of \emph{exact matches} (correct colour in the correct position), and
  \item White peg $w_t \in \{0, \ldots, L\}$ is the number of \emph{colour matches} (correct colour regardless of position), computed as $w_t = \sum_{c=0}^{C-1} \min\!\bigl(\text{count}(c, \mathbf{c}^*),\, \text{count}(c, \mathbf{g}_t)\bigr)$, where $\text{count}(c, \mathbf{v})$ is the number of occurrences of colour $c$ in vector $\mathbf{v}$.
\end{itemize}
Note that $w_t \geq b_t$ always holds, and colours are not double-counted: if the code contains $N$ repeats of a colour but the guess has $M > N$ entries of that colour, only $N$ contribute to $w_t$.
  
The episode proceeds for a fixed trajectory length of $T = 8$ guesses. A random initial guess is generated at the start of each episode and forms the first observation.

\subsubsection{Bayesian Oracle World Model}
\label{app:mm-bayes}

The Bayesian oracle world model maintains the exact posterior over secret codes by enumerating over all $C^L$ possible codes. Given a history of guesses and responses $\{(\mathbf{g}_1, b_1, w_1), \ldots, (\mathbf{g}_{t-1}, b_{t-1}, w_{t-1})\}$, the oracle computes the set of \emph{consistent codes}---those that would have produced the observed responses for all past guesses. Formally, a code $\mathbf{c}$ is consistent at step $t$ if for all $i < t$:    $\text{exact}(\mathbf{c}, \mathbf{g}_i) = b_i$  and $\text{colour}(\mathbf{c}, \mathbf{g}_i) = w_i$.
  
  The uniform prior over codes, combined with consistency filtering, yields the posterior probability of a newly observed response $(b_t, w_t)$ given a new guess $\mathbf{g}_t$ as:
  \begin{equation}
      p(b_t, w_t \mid h_t = [h_{t-1}, s_{t-1}, \mathbf{g}_t]) = \frac{|\{\mathbf{c} \in \mathcal{C}_t : \text{exact}(\mathbf{c}, \mathbf{g}_t) = b_t \wedge \text{colour}(\mathbf{c}, \mathbf{g}_t) = w_t\}|}{|\mathcal{C}_t|},
  \end{equation}
  where $\mathcal{C}_t$ denotes the set of codes consistent with all observations prior to step $t$.

\subsubsection{In-Context PFN World Model} \label{app:mm-worldmodel}

\paragraph{Data-Generating Process}
The PFN is trained on synthetically generated game trajectories. For each training example, a secret code is sampled uniformly at random, and then a sequence of $T_{\text{train}} = 50$ random guesses is generated. Each guess is drawn uniformly from $\{0, \ldots, C-1\}^L$, and the corresponding responses $(b_t, w_t)$ are computed by the environment.

\paragraph{Model Architecture}
The in-context world model is a causal Transformer decoder that processes interleaved guess and response tokens. 
Specifically, each guess $\mathbf{g}_t$ is one-hot encoded over the $C$ colours for each of the $L$ positions, yielding a vector of dimension $C \times L$, which is then projected to the model dimension $d$ via a linear layer. 
Each response $(b_t, w_t)$ is one-hot encoded as a concatenation of two vectors of dimension $(L+1)$ each, giving a total dimension of $2(L+1)$, which is projected to $d$ via a separate linear layer.
For a trajectory of $T$ steps, the input sequence is constructed by interleaving \emph{pairs} and \emph{isolated guesses}. A pair token is formed by summing the guess and response embeddings for the same timestep: $\mathbf{p}_t = \text{embed}_g(\mathbf{g}_t) + \text{embed}_r(b_t, w_t)$. The isolated guess token for the next step is $\text{embed}_g(\mathbf{g}_{t+1})$. These are interleaved as $[\mathbf{p}_1, \text{embed}_g(\mathbf{g}_2), \mathbf{p}_2, \text{embed}_g(\mathbf{g}_3), \ldots]$, resulting in a sequence of length $2(T - 1)$.

The sequence is processed by a causal (left-to-right) Transformer decoder with the model dimension $d = 512$, MLP hidden dimension = 1024, 12 layers, and 8 attention heads with head dimension = 64.
The Transformer output at every odd position (i.e., the isolated guess positions) is used to predict the corresponding response. A linear readout maps the Transformer output to $(L + 1)^2$ logits, representing the joint distribution over all possible $(b_t, w_t)$ pairs via a softmax. The target is the index $b_t \cdot (L + 1) + w_t$.

The in-context world model is trained to minimize the cross-entropy loss of predicting the response at each step. In each training step, a fresh batch of $B = 128$ random trajectories is generated on-the-fly (no fixed dataset).

\subsubsection{Policy}

The policy is trained using REINFORCE with a learned value baseline. At each training step, a batch of $B = 32$ trajectories of length $T = 8$ is collected by rolling out the current policy in the environment. The total loss for REINFORCE is:
\begin{equation}
  \mathcal{L} = \mathcal{L}_\pi + \lambda_V \mathcal{L}_V,
\end{equation}
where $\lambda_V = 0.5$. When using the in-context world model (PFN observer), an entropy bonus $-\lambda_H \mathcal{H}[\pi_\phi]$ with $\lambda_H = 0.01$ is added to the loss to encourage exploration. 

\paragraph{Model Architecture}

The policy model's backbone is a Griffin model with the same configuration as the one for GP. The guess $\mathbf{g}_t$ is one-hot encoded over the $C$ colours for each position, yielding a vector of dimension $C \times L$. The exact matches $b_t$ and colour matches $w_t$ are each one-hot encoded into vectors of dimension $L + 1$. These are concatenated into a single vector of dimension $C \cdot L + 2(L + 1)$ and projected to the model dimension $d$ via a linear layer.

The policy produces a value estimate and an action distribution. A linear layer maps the sequence model output at each step to a scalar value estimate $V(s_t)$. The action (a guess of $L$ colour values) is generated autoregressively, one position at a time, using a GRU cell. At each position $k \in \{1, \ldots, L\}$:
  \begin{enumerate}
      \item The GRU cell updates its hidden state using the current input (the sequence model output for $k=1$, or the embedding of the previously selected colour for $k > 1$).
      \item A linear readout produces logits over the $C$ colours for position $k$.
      \item A colour is sampled from $\text{Categorical}(\text{softmax}(\text{logits}_k / \tau))$, where $\tau$ is the temperature (set to $1.0$ during training).
      \item The selected colour is embedded and fed as input to the GRU for the next position.
  \end{enumerate}
The log-probability of the full action is the sum of the per-position log-probabilities: $\log \pi(a_t \mid s_t) = \sum_{k=1}^{L} \log \pi_k(a_t^{(k)} \mid s_t, a_t^{(1:k-1)})$.

\subsection{Alchemy}

The Alchemy environment is inspired by the latent-chemistry paradigm of \citet{wang2021alchemy}.
At the start of each episode, an environment is sampled with a hidden chemistry: a set of deterministic transition rules governing how potions transform stones. Stones are described by $D = 4$ discrete dimensions with value counts $(3, 3, 3, 4)$, yielding $3 \times 3 \times 3 \times 4 = 108$ distinct stone types. Potions are described by a single discrete dimension with $P = 6$ values. Each chemistry defines a fixed mapping from (initial\_stone, potion) pairs to resulting stones; if a given (stone, potion) pair has no matching rule, the stone is returned unchanged and the transition is marked as invalid. The full set of chemistries and their transition rules are stored as a precomputed transition table of shape $[N_{\text{envs}}, N_{\text{rules}}, 2D + P]$.

At each timestep $t$, the agent chooses an action $a_t = (i_t, p_t)$---an initial stone and a potion---and the environment returns an observation $(v_t,f_t)$ where $v_t \in \{0, 1\}$ indicates whether the transition was valid (i.e., a matching rule exists in the current chemistry), and
$f_t$ is the resulting stone (equal to $i_t$ when $v_t = 0$).

The episode proceeds for a fixed trajectory length of $T = 5$ steps. A random initial transition is sampled from the environment's transition table at the start of each episode and forms the first observation.

\subsubsection{Bayesian Oracle World Model}
Similar to Mastermind, the Bayesian oracle world model maintains the exact posterior over hidden chemistries by enumerating all $N_{\text{envs}}$ candidate environments. Given a history of observations, the oracle computes a \emph{consistency mask}: a candidate chemistry $\theta$ is consistent at step $t$ if, for all prior observations $(i_{t'}, p_{t'}, v_{t'}, f_{t'})$ with ${t'} < t$:
  \begin{enumerate}
    \item If $v_{t'} = 1$: chemistry $\theta$ contains a rule $(i_{t'}, p_{t'}) \mapsto f_{t'}$ that produces the observed output stone, and
    \item If $v_{t'} = 0$: chemistry $\theta$ has no rule matching the input pair $(i_{t'}, p_{t'})$.
  \end{enumerate}
  
The posterior probability of a new observation $s_t =(v_t, f_t)$ is computed as:
\begin{equation}
\rho(s_t \mid h_t = [h_{t-1}, s_{t-1}, (i_t, p_T)]) = \frac{\bigl|\{e \in \mathcal{E}_t : e \text{ is consistent with } (i_t, p_t, v_t, f_t)\}\bigr|}{|\mathcal{E}_t|},
\end{equation}
where $\mathcal{E}_t$ denotes the set of chemistries consistent with all observations prior to step $t$. 

This world model provides exact Bayesian inference but scales linearly in the number of candidate environments, which must be enumerated exhaustively. 

\subsubsection{In-Context PFN World Model}\label{app:alchemy_architecture}

The world model is a 12-layer causal Transformer with model dimension
$d{=}512$, 8 attention heads, and feed-forward dimension $1024$.
Each observation $o_t = (i_t, p_t, f_t)$ is split into an \emph{input token}
(encoding $i_t$ and $p_t$ via concatenated one-hot vectors, projected to
$\mathbb{R}^d$) and an \emph{output token} (encoding $f_t$ similarly).
These are interleaved into a sequence
$[\text{in}_1, \text{out}_1, \text{in}_2, \text{out}_2, \dots]$
of length $2T$ and processed with causal attention, so that the hidden state
at each input position $\text{in}_t$ attends to all tokens up to and
including $\text{out}_{t-1}$.

The output prediction head is an autoregressive GRU decoder that, starting
from the Transformer's hidden state at $\text{in}_t$, sequentially predicts
each component of the observation output: first a validity flag (2 categories),
then each of the four stone dimensions
(3, 3, 3, and 4 categories respectively).
At each step the GRU receives the embedding of the previous ground-truth
component (teacher forcing during training), updates its hidden state, and
produces logits over the current component's categories.
Invalid logit positions (where the category index exceeds the component's
vocabulary) are masked to $-\infty$ before the softmax.

The model is trained on sequences of $T_{\mathrm{train}}{=}50$ transitions
sampled with uniformly random actions from uniformly sampled environment IDs,
with batch size $B{=}64$ and the Adam optimizer at learning rate $10^{-4}$.

\paragraph{Data-Generating Process}
The PFN is trained on synthetically generated transition trajectories. For each training example, a chemistry (environment) is sampled uniformly at random, and then a sequence of $T_{\text{train}} = 20$ random transitions is generated. Each transition draws a random initial stone and potion uniformly from their respective value ranges, and the corresponding observation $(v_t, \textbf{s}_t^{\text{out}})$ is computed by the environment.

\paragraph{Model Architecture}
The in-context world model is a causal Transformer decoder that processes interleaved input (action) and output (observation) tokens. 
Each input token encodes the initial stone and potion via per-dimension one-hot vectors. The stone dimensions $(3, 3, 3, 4)$ and potion dimension $(6)$ are each one-hot encoded and concatenated, yielding a vector of dimension of 19, which is projected to the model dimension $d$ via a linear layer.
Each output token encodes the validity indicator $v_t$ and the final stone. The validity is one-hot encoded into a vector of dimension 2, and is concatenated with the final stone's one-hot into a 15-dimensional vector, which is projected to $d$ via a separate linear layer.
  
For a trajectory of $T$ steps, the input sequence is constructed by interleaving input and output embeddings. For each timestep $t$, the input embedding and output embedding are placed at positions $2t$ and $2t+1$, resulting in a sequence of length $2T$. The sequence is processed by a causal Transformer decoder with model dimension $d = 512$, MLP hidden dimension $F = 1024$, $N = 12$ layers, $H = 8$ attention heads with head dimension $d/H = 64$.

Similar to Mastermind's policy, the output prediction uses a GRU-based autoregressive decoder that conditions each output component's prediction on the previously observed ground-truth components (teacher forcing during training). 
The per-step loss is the sum of the per-component negative log-probabilities: $-\sum_{k=1}^{5} \log p(o_t^{(k)} \mid o_t^{(1:k-1)}, h_t)$. The in-context world model is trained to minimize the cross-entropy loss of predicting the output at each step. In each training step, a fresh batch of $B = 64$ random trajectories is generated on-the-fly (no fixed dataset), and the model is optimized using Adam with a learning rate of $10^{-4}$.

\subsubsection{Policy}

The policy is trained using PPO (Proximal Policy Optimization). At each training step, a batch of $B = 32$ trajectories of length $T = 5$ is collected by rolling out the current policy in the environment. The PPO objective uses $K = 4$ update epochs per batch of collected data. 
% The clipped surrogate loss is:
% \begin{equation}
% \mathcal{L}_\pi = -\frac{1}{T} \sum_{t=1}^{T} \min\!\Bigl(r_t(\theta)\, A_t,\; \text{clip}\bigl(r_t(\theta), 1 - \epsilon, 1 + \epsilon\bigr)\, A_t \Bigr),
% \end{equation}
% where $r_t(\theta) = \frac{\pi_\phi(a_t \mid s_t)}{\pi_{\phi{\text{old}}}(a_t \mid s_t)}$ is the probability ratio, $A_t = R - V_\phi(s_t)$ is the advantage (with $R$ the trajectory-level reward from the world model and $V_\phi$ the learned value baseline), and $\epsilon = 0.2$ is the clipping parameter. The value loss is:
% \begin{equation}
% \mathcal{L}_V = \frac{1}{T} \sum_{t=1}^{T} \bigl(R - V_\phi(s_t)\bigr)^2.
% \end{equation}
% The total loss combines the policy, value, and entropy terms:
% \begin{equation}
% \mathcal{L} = \mathcal{L}_\pi + \lambda_V \mathcal{L}_V - \lambda_H \mathcal{H}[\pi_\phi],
% \end{equation}
% where $\lambda_V = 0.5$ and $\lambda_H = 0.01$.

\paragraph{Model Architecture}

The policy is the same sequence model as used in Mastermind, with a Griffin backbone model \citep{de2024griffin}, and a GRU to predict the final output components autoregressively conditioned on the Griffin's output.
The log-probability of the full action is the sum of the per-component log-probabilities: $\log \pi(a_t \mid s_t) = \sum_{k=1}^{5} \log \pi_k(a_t^{(k)} \mid s_t, a_t^{(1:k-1)})$.

\end{appendices}

% \newpage
% \input{checklist.tex}

%%%%%%%%%%%%%%%%%%%%%%%%%%%%
% Examples
%%%%%%%%%%%%%%%%%%%%%%%%%%%%

% Figure:

% \begin{figure}[ht]
%     \centering
%     \includegraphics[width=\linewidth]{figures/neural_activation_space.jpg}
%     \caption{\textbf{Figure title.} \textbf{a.} Caption for first sub-figure. \textbf{b.} The caption for the second sub-figure.}
%     \label{fig:figure_label}
% \end{figure}

% Table:

% \begin{table}
%     \caption{Sample table title}
%     \centering
%     \begin{tabular}{lll}
%         \toprule
%         \multicolumn{2}{c}{Part}                   \\
%         \cmidrule(r){1-2}
%         Name     & Description     & Size ($\mu$m) \\
%         \midrule
%         Dendrite & Input terminal  & $\sim$100     \\
%         Axon     & Output terminal & $\sim$10      \\
%         Soma     & Cell body       & up to $10^6$  \\
%         \bottomrule
%     \end{tabular}
%     \label{tab:table_label}
% \end{table}

% Box:

% \begin{mybox}[label=box:box_label]{Box title}
% Box contents.
% \end{mybox}

\end{document}